\documentclass[11pt]{article}
\usepackage[margin=1in]{geometry}

\usepackage[T1]{fontenc}

\usepackage[dvipsnames]{xcolor}
\usepackage{algorithm}
\usepackage[noend]{algpseudocode}

\algnewcommand{\LComment}[1]{\Statex \textcolor{gray}{\(\triangleright\) \textit{#1}}}

\usepackage[hyphens]{url}
\usepackage{natbib}
\usepackage{amssymb}

\usepackage{enumitem}
\usepackage{footnote}
\usepackage[bottom]{footmisc}
\usepackage{tablefootnote}
\usepackage{booktabs}
\usepackage{makecell}
\usepackage{subcaption}
\usepackage{hyperref}
\newcolumntype{C}[1]{>{\centering\arraybackslash}p{#1}}
\usepackage{graphicx}%
\usepackage{multirow}%
\usepackage{amsmath,amssymb,amsfonts}%
\usepackage{amsthm}%
\usepackage{mathrsfs}%
\usepackage[title]{appendix}%
\usepackage{xcolor}%
\usepackage{textcomp}%
\usepackage{manyfoot}%
\usepackage{booktabs}%
\usepackage{algorithm}%
\usepackage{algorithmicx}%
\usepackage{algpseudocode}%
\usepackage{listings}%
\usepackage{pifont}
\usepackage[font=small,labelfont=bf,width=5.7in]{caption}

\theoremstyle{thmstyleone}%
\newtheorem{theorem}{Theorem}
\newtheorem{proposition}[theorem]{Proposition}%

\newtheorem{assumption}{Assumption}

\theoremstyle{thmstyletwo}%

\theoremstyle{thmstylethree}%
\newtheorem{definition}{Definition}%

\makeatletter
\newcommand{\phantomfill}[2]{{\mathpalette\mask@{{#1}{#2}}}}
\newcommand{\mask@}[2]{\mask@@{#1}#2}
\newcommand{\mask@@}[3]{%
  \settowidth{\dimen@}{$\m@th#1#2$}%
  \makebox[\dimen@]{$\m@th#1#3$}%
}
\makeatother

\newcommand{\eg}{\textit{e.g.}}
\newcommand{\ie}{\textit{i.e.}}
\newcommand{\methodname}{\textrm{NPC}}

\newcommand{\tick}{\ding{51}} 
\newcommand{\cross}{\ding{55}} 

\usepackage{siunitx}[=v2]
\usepackage{authblk}

\begin{document}

\title{Neural Probabilistic Circuits: Enabling Compositional and Interpretable Predictions through Logical Reasoning}
\date{}

\author[1]{Weixin Chen\thanks{Equal contributions.}}
\author[1]{Simon Yu$^*$}
\author[2]{Huajie Shao}
\author[1]{Lui Sha}
\author[1]{Han Zhao\thanks{Corresponding author.}}

\affil[1]{University of Illinois Urbana-Champaign}
\affil[2]{College of William and Mary}

\affil[ ]{\protect\\{\small\texttt{\{\href{mailto:weixinc2@illinois.edu}{weixinc2},\href{mailto:jundayu2@illinois.edu}{jundayu2},\href{mailto:lrs@illinois.edu}{lrs},\href{mailto:hanzhao@illinois.edu}{hanzhao}\}@illinois.edu, \{\href{mailto:hshao@wm.edu}{hshao}\}@wm.edu}}}

\maketitle

\abstract{
End-to-end deep neural networks have achieved remarkable success across various domains but are often criticized for their lack of interpretability. While post hoc explanation methods attempt to address this issue, they often fail to accurately represent these black-box models, resulting in misleading or incomplete explanations. To overcome these challenges, we propose an inherently transparent model architecture called Neural Probabilistic Circuits (\methodname s), which enable compositional and interpretable predictions through logical reasoning. In particular, an \methodname~consists of two modules: an attribute recognition model, which predicts probabilities for various attributes, and a task predictor built on a probabilistic circuit, which enables logical reasoning over recognized attributes to make class predictions. To train \methodname s, we introduce a three-stage training algorithm comprising attribute recognition, circuit construction, and joint optimization. Moreover, we theoretically demonstrate that an \methodname 's error is upper-bounded by a linear combination of the errors from its modules. To further demonstrate the interpretability of \methodname, we provide both the most probable explanations and the counterfactual explanations. Empirical results on four benchmark datasets show that \methodname s strike a balance between interpretability and performance, achieving results competitive even with those of end-to-end black-box models while providing enhanced interpretability.

}

\section{Introduction}\label{sec:intro}
End-to-end deep neural networks (DNNs)~\citep{alexnet, resnet, transformer, bert} have demonstrated remarkable success across various domains~\citep{speech, translation, segmentation}. However, many of them are black-box models containing complex operators, making it hard to interpret and understand how a decision was made. Although many efforts~\citep{lime, shap, gradcam} have been made to explain a model's decision in a post hoc manner, \citet{explanation1, explanation2, explanation3, stop} show that these explanations are oftentimes not reliable as the explanation model might loosely approximate the underlying model. For example, the explanation model exhibits similar performance to the black-box model but yet relies on entirely different features. Such discrepancy between the explanation model and the black-box model could lead to misleading explanations, \eg, attributing the decision to irrelevant features or missing out important features. Misleading explanations are particularly concerning in high-stake applications such as medical analysis~\citep{hou2024self, liu2023explainable} and legal justice~\citep{richmond2024explainable, deeks2019judicial}. Rather than introducing post hoc explanations to explain a black-box model, \citet{stop} argues that one should create an interpretable model in the first place where each component is designed with a distinct purpose, facilitating an interpretable prediction.

Concept bottleneck models (CBMs)~\citep{cbm} aim to enhance interpretability by introducing high-level, human-understandable concepts, such as ``red color'' and ``round shape'', as an intermediate bottleneck, which decomposes a model into two modules: a concept recognition model and a task predictor. The neural-network-based concept recognition model maps the input image to probabilities associated with various concepts. Using these probabilities, the task predictor, typically a linear predictor, produces the probabilities for the various classes. Since the final prediction (\ie, the class with the highest probability) can be interpreted in terms of these concepts, the model's decision-making process becomes more intuitive for humans to understand.
To improve performance on downstream tasks, methods like CEM~\citep{cem}, ProbCBM~\citep{probcbm}, and others~\citep{concept_emb1, concept_emb2} change the outputs of the concept recognition model from concept probabilities to concept embeddings. While boosting task performance, such approaches significantly compromise interpretability since the dimensions within concept embeddings lack semantic meanings.
On the other hand, to further improve interpretability, some approaches~\citep{dcr, soft_tree_cbm, len} propose architectures for the task predictor that incorporate logical rules, allowing task predictions to be explicitly explained through these rules. However, these logical rules are usually learned from data rather than being predefined by humans, limiting our ability to integrate prior domain knowledge into the model. 
Additionally, there is currently no theoretical guarantee regarding the performance of the overall model, obscuring the relationship between overall performance and that of the concept recognition model or the task predictor.

To address these challenges, we propose a novel model architecture called Neural Probabilistic Circuits (\methodname s), which enable compositional and interpretable predictions through logical reasoning. An \methodname~comprises two modules: an attribute recognition model and a task predictor. Unlike existing approaches that primarily focus on numerous binary concepts (\eg, ``red color'', ``yellow color''), we introduce a higher-level categorical characteristic called attributes, which describe the types of concepts (\eg, ``color''). 
This approach reduces the need for additional concept selection or pruning to improve model efficiency~\citep{len, entropy_len, intervention_aware}, while also achieving better performance in concept recognition.
Given an input image, the neural-network-based attribute recognition model produces probability vectors for various attributes, with each vector representing the likelihood of various values for the corresponding attribute. These probability vectors serve as inputs to the task predictor, which is implemented using a probabilistic circuit.
Probabilistic circuits~\citep{spn, cccp, zhao2015relationship, choi2020probabilistic}, a type of graphical models~\citep{pgm}, aim to learn the joint distribution over input variables, in our case, attribute variables and the class variable. During learning, probabilistic circuits embed within their structures and parameters either implicit logical rules learned from data or explicit logical rules predefined by humans. The circuits enable tractable probabilistic reasoning tasks such as joint, marginal, and conditional inferences, thereby revealing relations among the attributes and classes.
By leveraging these relations, \methodname s can reason over outputs from the attribute recognition model to infer the most probable class. Specifically, the prediction score for a given class is the sum of the likelihood of each combination of attribute values weighted by their relevance to the class. As usual, the final prediction corresponds to the class with the highest score.

Given the compositional nature of \methodname s, we propose a three-stage training algorithm. Specifically, the whole procedure involves the following stages:
1) Attribute recognition: We begin by training the attribute recognition model within a multi-task learning framework~\citep{caruana1997multitask, ruder2017overview}.
2) Circuit construction: Next, we construct the circuit using two distinct approaches:
i) Data-driven approach learns the circuit's structure and optimizes its parameters based on data, allowing the underlying logical rules to be embedded within the circuit.
ii) Knowledge-injected approach manually designs the circuit's structure and assigns its parameters to ensure that human-predefined logical rules are explicitly encoded within the circuit.
3) Joint optimization: Finally, the two modules are jointly optimized in an end-to-end manner to further enhance the overall model's performance on downstream tasks.

To provide theoretical guarantees regarding the performance of the overall model, we demonstrate that, due to its compositional nature and the use of probabilistic circuits, \methodname s exhibit a compositional error bound—the error of the overall model is upper-bounded by a linear combination of the errors from the various modules.

In addition, we provide various types of explanations to make it easier for humans to understand \methodname's predictions:
1) Most Probable Explanation (MPE) identifies the combination of attribute values that contributes most significantly to the predicted class.
2) Counterfactual Explanation (CE) answers the question: Would the model have made the correct prediction had the likelihoods of certain attribute values been adjusted?

Empirical results on four image classification datasets demonstrate \methodname's ability to strike an impressive balance between interpretability and performance on downstream tasks. In particular, \methodname~outperforms three representative concept-based models and delivers results competitive even with those of an end-to-end deep neural network. Additionally, we perform extensive ablation studies to investigate the advantages of integrating attributes over concepts, along with the impacts of attribute selections, predictor design approaches, and joint optimization.

Our main contributions are as follows,
\begin{enumerate}
    \item We introduce Neural Probabilistic Circuits (\methodname s), a novel model architecture that combines a neural-network-based attribute recognition model and a probabilistic-circuit-based task predictor, enabling compositional and interpretable predictions through logical reasoning.
    \item We develop a three-stage training algorithm for \methodname s, consisting of 1) attribute recognition through multi-task learning, 2) circuit construction via both data-driven and knowledge-injected approaches, and 3) end-to-end joint optimization.
    \item To the best of our knowledge, we are the first to provide a theoretical guarantee on the performance of the compositional bottleneck models, which shows that \methodname~'s error can be upper-bounded by a linear combination of the errors from its modules.
    \item We provide various types of explanations to facilitate human understanding of \methodname s' predictions, including most probable explanations and counterfactual explanations.
    \item We empirically show that \methodname~demonstrates competitive performance in image classification tasks while providing enhanced interpretability.
\end{enumerate}

\section{Preliminaries}\label{sec:prelim}
Probabilistic circuits are a class of graphical models that is used to express a joint distribution over a set of random variables $Z_{1:N}$.
A probabilistic circuit $f_S$ (henceforth simply referred to as a \textit{circuit}) consists of a rooted directed acyclic graph where leaf nodes are univariate indicators of categorical variables\footnote{We mainly focus on probabilistic circuits over categorical random variables. An extension to the continuous ones is standard.} (\ie, $\mathbb{I}(Z_i=z_i),~z_i\in \mathcal{Z}_i,~i\in [N]$) and internal nodes consist of sum nodes and product nodes.
Each sum node computes a weighted sum of its children, and each product node computes a product of its children.
In an unnormalized circuit, the root node outputs the unnormalized joint probability over variables.
Any unnormalized circuit can be transformed into an equivalent, normalized circuit via weight updating \cite{PeharzTPD15, ZhaoMP15}. Hence, without loss of generality, we always assume that $f_S$ is normalized; thus, $f_S(z_{1:N}) = \Pr(Z_{1:N}=z_{1:N})$. Readers are referred to~\citet{sanchez2021sum} for more details on circuits.

In circuits, the \textit{scope} of a node is defined as the set of variables that have indicators among the node's descendants, which can be computed recursively—if $v$ is a leaf node, say, an indicator over $Z_i$, then $\text{scope}(v) = \{Z_i\}$; otherwise, $\text{scope}(v)=\cup_{\tilde{v} \in \text{children}(v)} ~\text{scope}(\tilde{v})$.
A circuit is \textit{smooth} iff each sum node has children with identical scope. A circuit is \textit{decomposable} iff each product node has children with disjoint scopes.
If a circuit is smooth and decomposable, then any marginal probability can be computed by setting the leaf nodes corresponding to the marginalized variables to 1. Consequently, inferences are efficient in a circuit as any joint, marginal, or conditional inference can be computed by at most two passes in a circuit. For instance, $\Pr(Z_1=z_1 \mid Z_{2:N}=z_{2:N}) = \frac{\Pr(Z_{1:N}=z_{1:N})}{\Pr(Z_{2:N}=z_{2:N})} = \frac{f_S(Z_{1:N}=z_{1:N})}{f_S(Z_1=\emptyset, Z_{2:N}=z_{2:N})}$ where $Z_1=\emptyset$ implies $\mathbb{I}(Z_1=\tilde{z}_1)=1, \forall \tilde{z}_1 \in \mathcal{Z}_1$; thus, computing a conditional probability only requires two forward processes in a circuit, each in linear time \textit{w.r.t.} its size. In this paper, we focus on smooth and decomposable circuits.

\section{Neural Probabilistic Circuits}\label{sec:method}
In this section, we introduce Neural Probabilistic Circuits (\methodname s). We begin by describing the model architecture and the inference process, illustrating how \methodname s enable compositional and interpretable predictions through logical reasoning (Section \ref{sec:model_architecture}). Next, we elaborate on the three-stage training algorithm for training \methodname s. In particular, we propose two distinct approaches for building circuits: a data-driven approach and a knowledge-injected approach (Section \ref{sec:training_algorithm}). Finally, we provide a theoretical analysis establishing the relationship between the error of the overall model and those of its individual modules (Section \ref{sec:theoretical_analysis}).

\subsection{Model Architecture and Inference} \label{sec:model_architecture}
Figure \ref{fig:framework} presents an overview of an \methodname, which consists of an attribute recognition model and a task predictor.
\begin{figure}[t]
    \centering
    \includegraphics[width=5.7in]{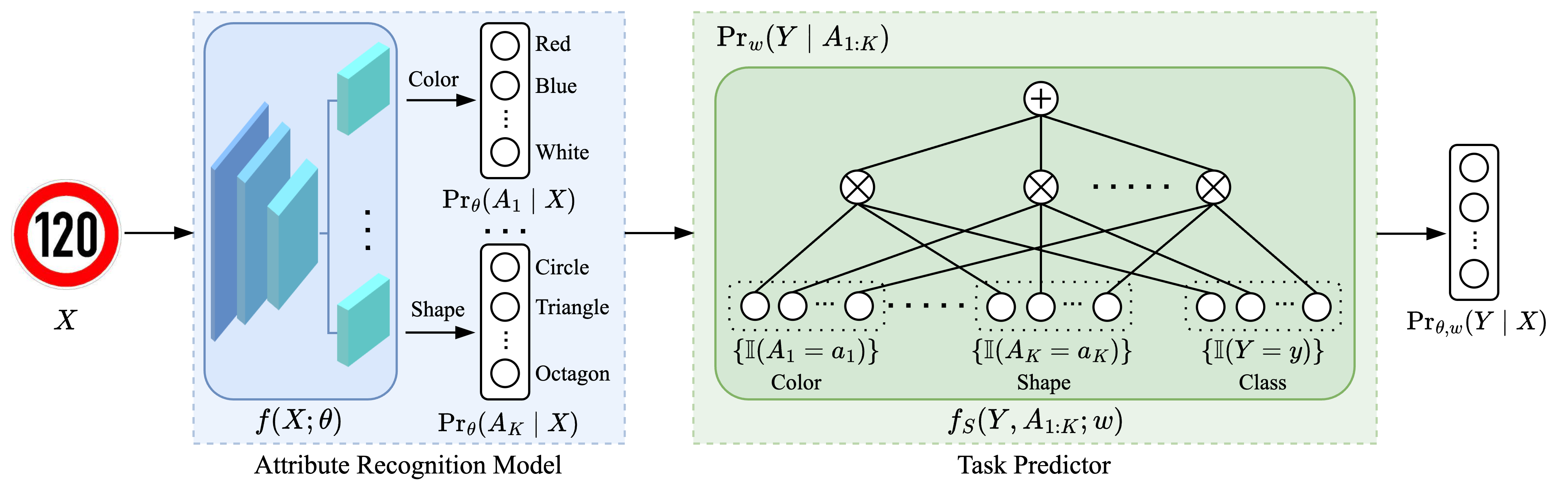}
    \caption{\small Model architecture of an \methodname, consisting of an attribute recognition model and a task predictor.
    The attribute recognition model is a neural network $f(X;\theta)$ which takes an image $X$ as input and outputs $K$ probability vectors $\{\Pr_{\theta}(A_k\mid X)\}_{k=1}^K$.
    The task predictor is a probabilistic circuit $f_S(Y, A_{1:K}; w)$ taking an instance of attributes as input and providing the conditional probability $\Pr_w(Y\mid A_{1:K})$. 
    By leveraging these relations between classes and attributes alongside the probability distributions of various attributes, \methodname~produces the probability vector $\Pr_{\theta, w}(Y\mid X)$.
    \label{fig:framework}}
\end{figure}
The attribute recognition model is a neural network that processes an input image to identify its high-level visual attributes, such as color and shape.
The task predictor is a (normalized) probabilistic circuit that models the joint distribution over attributes and classes, embedding either implicit or explicit logical rules within its structure and parameters during learning. The circuit enables efficient probabilistic reasoning, including joint, marginal, and conditional inferences. Specifically, given a particular assignment of attributes, the circuit can infer the probability of a specific class. By leveraging these conditional dependencies alongside the probability distributions of the various attributes (\ie, outputs from the attribute recognition model), \methodname~produces the probabilities of the image belonging to various classes. The class with the highest probability is recognized as the predicted class.

Formally, let $X\in\mathcal{X}, A_k\in\mathcal{A}_k, Y\in\mathcal{Y}$ denote the input variable, the $k$-th attribute variable, and the class variable. The variables' instantiations are represented by $x, a_k, y$, respectively. In particular, we consider $K$ attributes, \ie, $A_1, \ldots, A_K$ (or $A_{1:K}$ in short). Each attribute $A_k$ has $q_k$ possible values, \ie, $\vert \mathcal{A}_k \vert = q_k$.
The attribute recognition model $f(X;\theta)$ is parameterized by $\theta$. Given an input instance $x$, the model outputs $K$ probability vectors. The $k$-th probability vector, denoted as $f_k(x; \theta) \in \mathbb{R}^{q_k}$, shows the probabilities of $x$'s $k$-th attribute taking different values $a_k$, \ie, $[f_k(x; \theta)]_{a_k} = \Pr_{\theta}\left( A_k=a_k \mid X=x \right)$.
The task predictor $f_S(Y, A_{1:K}; w)$ is a probabilistic circuit with structure $S$ and parameters $w$, which models the joint distribution over $Y, A_{1:K}$. Specifically, when taking an instance of attributes $a_{1:K}$ and a class label $y$ as input, the circuit outputs the joint probability $\Pr_w(Y=y, A_{1:K} = a_{1:K})$. By the virtue of its efficient inferences, the circuit also supports efficient conditional queries, \eg, $\Pr_w\left(Y=y\mid A_{1:K} = a_{1:K} \right) = f_S(y, a_{1:K}; w)/f_S(\emptyset, a_{1:K}; w)$.

Prior to describing how an \methodname~predicts a class, we make the following mild assumptions on the selected attributes.

\begin{assumption}[Sufficient Attributes]
    \label{assumption1}
    Given the attributes, the class label is conditionally independent of the input, \ie, $Y\perp X\mid A_1,\ldots, A_K$.
\end{assumption}

\begin{assumption}[Complete Information]
    \label{assumption2}
    Given any input, all attributes are conditionally mutually independent, \ie, $A_1\perp A_2\perp\cdots\perp A_K \mid X$.
\end{assumption}

\noindent\textbf{Remarks.}
The first assumption essentially assumes that the attributes are sufficient to infer the class label of interest. The second assumption assumes the input contains complete information regarding the attributes such that they are conditionally mutually independent.
These assumptions are mild and often hold in practice. For instance, in the context of traffic signs, if the attributes include the shape (\eg, circle), color (\eg, red), and symbol (\eg, slash) of a sign, they collectively provide enough information to infer the class label (\eg, no entry) without requiring additional details from the raw image. On the other hand, the raw image fully encodes the attributes (\eg, shape, color, and symbol). Once the input is observed, these attributes can be independently determined.

Under Assumption \ref{assumption1} and \ref{assumption2}, an \methodname~outputs the probability of an input $x$ being a class $y$ as follows,
{
\small
\begin{align}
    \operatorname{Pr}_{\theta, w}\left(Y=y \mid X=x\right)
     & =\sum_{a_{1:K}}\operatorname{Pr}_{w}\left(Y=y \mid A_{1:K} = a_{1:K}, X=x\right) \cdot \operatorname{Pr}_{\theta}\left(A_{1:K} = a_{1:K} \mid X=x\right) \notag                                                                                           \\
     & =\sum_{a_{1:K}}\underbrace{\operatorname{Pr}_{w}\left(Y=y \mid A_{1:K} = a_{1:K}\right)}_{\text{task predictor}} \cdot \underbrace{\prod_{k=1}^K \operatorname{Pr}_{\theta}\left(A_k=a_k \mid X=x\right)}_{\text{attribute recognition model}}.\label{eq:inference}
\end{align}
}
Equation (\ref{eq:inference}) is derived using the assumptions, and the two interior terms are given by the circuit-based task predictor and the attribute recognition model, respectively.
Subsequently, the predicted class is the one with the largest probability, \ie, $\hat{y} = \arg\max_{y\in \mathcal{Y}}~\Pr_{\theta, w}(Y=y \mid X=x)$.

In summary, we propose a novel model architecture for image recognition tasks. The architecture is interpretable by design, thanks to the integration of an attribute bottleneck and the probabilistic semantics of probabilistic circuits. Together, these modules enable predictions which can be interpreted using the likelihood of different attributes and the conditional dependencies between attributes and classes.

\subsection{Three-Stage Training Algorithm} \label{sec:training_algorithm}
In this section, we will propose a three-stage training algorithm for \methodname s comprising the following stages:  1) attribute recognition through multi-task learning (Section \ref{sec:attribute_learning}), 2) circuit construction via both data-driven and knowledge-injected approaches (Section \ref{sec:relation_learning}), and 3) joint optimization (Section \ref{sec:joint_learning}).

\subsubsection{Attribute Recognition} \label{sec:attribute_learning}


We aim to train the attribute recognition model $f(X;\theta)$ such that each attribute is recognized well. To this end, we adopt a multi-task learning framework~\citep{zhang2021survey}, where each task is to recognize a particular attribute. Specifically, we use the cross-entropy loss for each task and assign weights to the task losses based on the size of the corresponding attribute space. These weights normalize the task losses, preventing certain tasks from dominating the training process~\citep{kendall2018multi, gregoire2024sample, wang2020multi}.
The overall training loss for attribute recognition is defined as follows,
{
\small
\begin{equation}
    \mathcal{L}_{\text{Attribute}}(\theta; D) = -\frac{1}{K}\sum_k\frac{1}{\log q_k} \left( \frac{1}{|D|} \sum_{x\in D}g_k^T(x) f_k(x;\theta) \right).\label{eq:attribute_loss}
\end{equation}
}
The term inside the parentheses represents the mean cross-entropy loss over the training dataset $D$, where $f_k(x;\theta)\in \mathbb{R}^{q_k}$ and $g_k(x)\in \mathbb{R}^{q_k}$ are, respectively, the output vector and the label vector corresponding to the $k$-th attribute. Specifically, $g_k(x)$ is a probability vector that sums to one, with each entry representing the ground-truth likelihood of $x$ having a particular attribute value. For instance, for the color attribute, an image of a polar bear would have a one-hot label vector with 1 representing ``white''. In contrast, an image of a zebra might have a probabilistic label vector, with 0.5 representing ``black'' and 0.5 representing ``white''.

\subsubsection{Circuit Construction} \label{sec:relation_learning}
We aim to construct a probabilistic circuit $f_S(Y, A_{1:K}; w)$ that models the joint distribution over $Y, A_{1:K}$. To achieve this, we propose two distinct approaches for building the circuit's structure and parameters: a data-driven approach and a knowledge-injected approach.

\paragraph{Data-Driven Approach}
This approach \textit{learns a circuit's structure and optimizes its parameters} using data in the form of $\{(y, a_{1:K})\}$. As described above, the training dataset is defined as $D=\{(x, g_{1:K}(x), y)\}$, where each attribute label is represented as a probability vector rather than a single value. For each data, we sample an $a_k$ from the categorical distribution defined by $g_k(x)$, \ie, $a_k \sim \text{Categorical}(g_k(x)),~k\in[K]$, which results in a processed dataset $\bar{D}=\{(x, y, a_{1:K})\}$.
\textbf{1) Structure Learning}: LearnSPN \cite{learnspn} is a mainstream algorithm for learning a circuit's structure from data. LearnSPN recursively identifies independent groups to create product nodes, clusters data to form sum nodes, and assigns single variables as leaf nodes. In our approach, we apply LearnSPN on $\bar{D}$ to derive a structure tailored to the observed data.
\textbf{2) Parameter Learning}: Given the learned structure, optimizing the circuit's weights (\ie, the weights of edges emanating from sum nodes) is framed as a maximum likelihood estimation (MLE) problem, with the following loss function:
{
\small
\begin{equation*}
    \mathcal{L}_{\text{MLE}}(w; D) = -\sum_{(y, a_{1:K})\in \bar{D}} \log f_S\left(y, a_{1:K}; w \right).
\end{equation*}
}
We adopt the widely used CCCP algorthim~\cite{cccp} which iteratively applies multiplicative weight updates on $w$ to minimize $\mathcal{L}_{\text{MLE}}(w; \bar{D})$. CCCP is guaranteed to converge monotonously.
Overall, with the learned structure and optimized parameters, the circuit captures the underlying logical rules present in the observed data, thus effectively modeling the joint distribution over attributes and classes.

\paragraph{Knowledge-Injected Approach}
Incorporating domain knowledge into a model helps ensure that its behavior aligns with the human's understanding of the domain. In practice, domain knowledge can be represented as a set of weighted logical rules. These rules are usually derived by observing patterns in existing samples, with the weight of each rule reflecting the frequency with which the rule holds true among the observed data. For instance, in the context of traffic signs, a rule could be: $\mathbb{I}(\text{shape}=\text{circle}) \wedge \mathbb{I}(\text{color}=\text{red}) \wedge \mathbb{I}(\text{symbol}=\text{slash}) \wedge \mathbb{I}(\text{class}=\text{no entry})$, as this holds true for an image of a ``no entry'' sign. 
The rules summarized for various tasks, along with the process of establishing them, are detailed in Appendix \ref{sec:rule}. 
By leveraging these rules, the knowledge-injected approach \textit{manually constructs a circuit's structure and assigns its parameters} to model the joint distribution over $Y, A_{1:K}$.
\textbf{1) Structure Construction}: We construct the structure of the circuit to be a depth-2 weighted sum-of-products formula. More specifically, consider $L$ rules of the form $\{r_l := \mathbb{I}(A_1=a_1^l) \wedge \ldots \wedge \mathbb{I}(A_K=a_K^l) \wedge \mathbb{I}(Y=y^l)\}_{l=1}^L$. The structure is constructed as follows: 
i) A set of leaf nodes is created to represent the indicator variables of $Y, A_{1:K}$.
ii) Based on these leaf nodes, a layer of product nodes is built, where each product node is associated with a rule. Specifically, the $l$-th product node connects to the leaf nodes that represent the conditions in the rule $r_l$, \ie, $\mathbb{I}(A_1=a_1^l), \ldots, \mathbb{I}(A_K=a_K^l), \mathbb{I}(Y=y^l)$. 
iii) A single sum node, which serves as the root node of the circuit, is placed above the product node layer. This sum node aggregates the outputs of all product nodes.
\textbf{2) Parameter Assignment}: The parameters in the circuit refer to the weights of edges connecting the product nodes to the sum node. The weight for the $l$-th edge is assigned as the frequency of the rule $r_l$.
An example of a circuit constructed using this approach is illustrated in Figure \ref{fig:manual_circuit}.
Through these two steps, the human-predefined logical rules are manually encoded into the circuit's structure and parameters. Proposition \ref{prop:joint} ensures that the output of the circuit's root node represents the empirical joint probability over attributes and classes.

\begin{figure}[tp]
    \centering
    \includegraphics[width=2.9in]{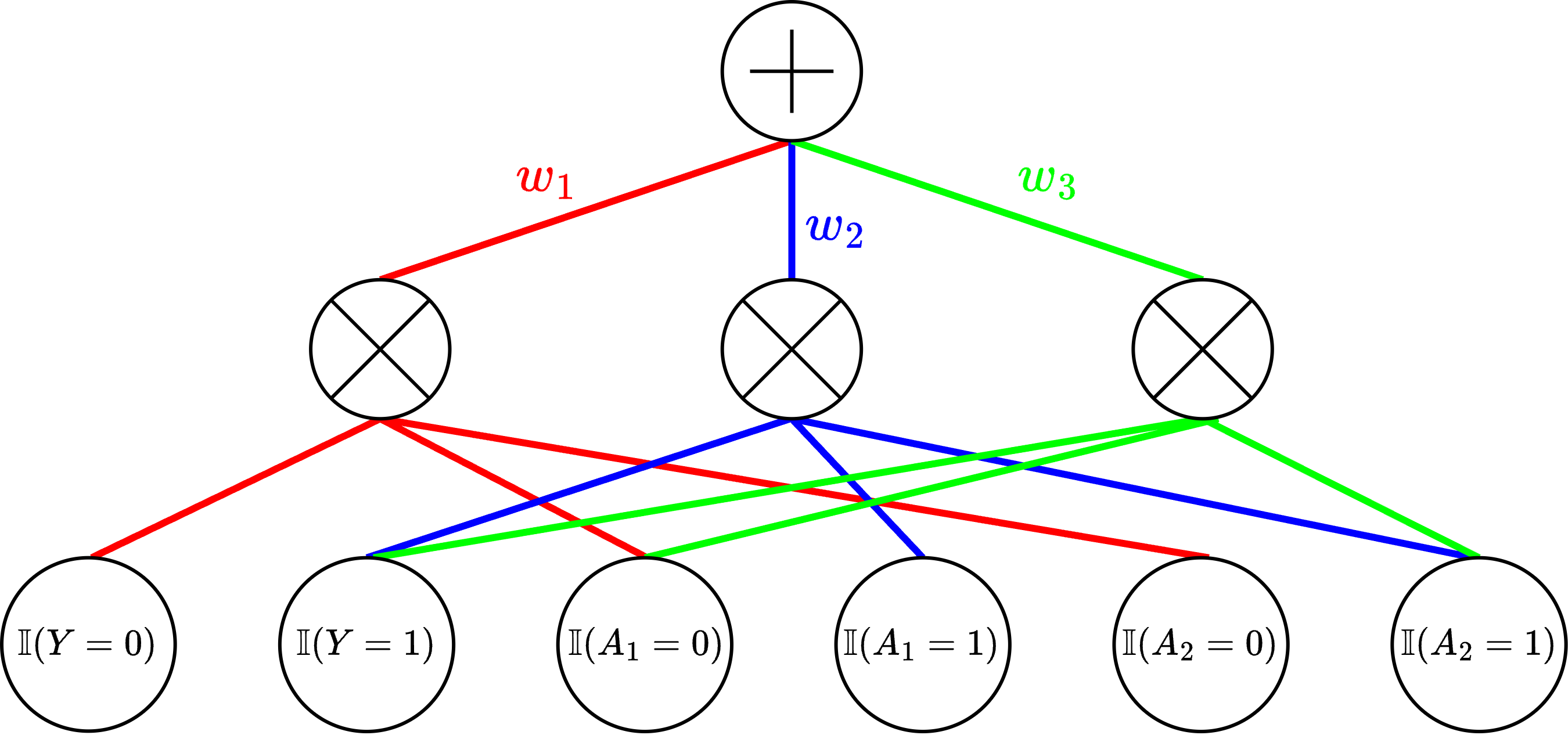}
    \caption{Illustration of a probabilistic circuit constructed by the knowledge-injected approach. The circuit encodes the set of weighted logical rules: $\{w_1 \cdot \left(\mathbb{I}\left(A_1=0\right) \wedge \mathbb{I}\left(A_2=0\right) \wedge \mathbb{I}\left(Y=0\right) \right),~w_2 \cdot \left(\mathbb{I}(A_1=1) \wedge \mathbb{I}(A_2=1) \wedge \mathbb{I}(Y=1) \right),~w_3 \cdot \left( \mathbb{I}\left(A_1=0\right) \wedge \mathbb{I}\left(A_2=1\right) \wedge \mathbb{I}\left(Y=1\right)\right)\}$.}
    \label{fig:manual_circuit}
\end{figure}

\begin{proposition}
The circuit constructed using the knowledge-injected approach models the empirical joint distribution over attributes and classes. Specifically, the output of the root node represents the empirical joint probability of attributes and classes.
\label{prop:joint}
\end{proposition}

\begin{proof}
Let $R$ denote the rule variable. For a given instance $(y, a_{1:K})$, the output of the root node is computed as,
{
\small
\begin{align*}
    f_S(y, a_{1:K}; w) &= \sum_l \Pr_{w}(R=r_l)\cdot\mathbb{I}(a_1=a_1^l)\ldots\mathbb{I}(a_K=a_K^l)\cdot\mathbb{I}(y=y^l) \\
    &= \sum_l \Pr_{w}(R=r_l)\cdot\Pr_{w}(A_{1:K}=a_{1:K}, Y=y \mid R=r_l) \\
    &= \Pr_{w}(A_{1:K}=a_{1:K}, Y=y),
\end{align*}
}
where $\Pr_{w}(R=r_l)$ is the weight for the $l$-th edge and $\mathbb{I}(a_1=a_1^l)\ldots\mathbb{I}(a_K=a_K^l)\cdot\mathbb{I}(y=y^l)$ is the output of the $l$-th product node. The condition $R=r_l$ indicates that the rule $\mathbb{I}(A_1=a_1^l) \wedge \ldots \wedge \mathbb{I}(A_K=a_K^l) \wedge \mathbb{I}(Y=y^l)$ is satisfied.
\end{proof}

\subsubsection{Joint Optimization} \label{sec:joint_learning}
Thanks to the differentiability of circuits, \methodname s can be fine-tuned in an end-to-end manner to further improve the performance of the overall model on downstream tasks. Specifically, the loss function is defined as follows,
{
\small
\begin{equation}
    \mathcal{L}_{\text{Joint}}(\theta, w; (x, y)) = -\sum_{(x, y)\in D}\log \operatorname{Pr}_{\theta, w} (Y=y \mid X=x). \label{eq:joint}
\end{equation}
}
To optimize this loss, we simply employ the stochastic gradient descent algorithm for updating $\theta$, while the projected gradient descent algorithm is used to update $w$ to ensure the positivity of the circuit weights. The detailed optimization process is provided in Appendix \ref{sec:joint}.

\subsection{Theoretical Analysis}
\label{sec:theoretical_analysis}
In this section, we present an error analysis for \methodname s to understand how the performance of individual modules affects that of the overall model. Given the overall model and the attribute recognition model are discriminative models, while the probabilistic circuit is a generative model, we define the following errors:
\textbf{1) Error of the overall model}: $\epsilon_{\theta, w} := \mathbb{E}_{X}\left[d_{\text{TV}}(\Pr_{\theta, w}(Y\mid X), \Pr(Y\mid X)) \right]$, which represents the expected total variance distance between the learned and true conditional distributions of $Y$ given $X$.
\textbf{2) Error of the attribute recognition model}: $\epsilon_{\theta} := \mathbb{E}_{X}\left[d_{\text{TV}}(\Pr_{\theta}(A_{1:K}\mid X), \Pr(A_{1:K}\mid X)) \right]$, which quantifies the expected total variation distance between the learned and true conditional distributions of the attributes $A_{1:K}$ given $X$. Additionally, we define $\epsilon_{\theta}^k := \mathbb{E}_{X}\left[d_{\text{TV}}(\Pr_{\theta}(A_k\mid X), \Pr(A_k\mid X)) \right]$ as the error for each individual attribute $A_k$.
\textbf{3) Error of the probabilistic circuit}: $\epsilon_w := d_{\text{TV}}(\Pr_{w}(Y, A_{1:K}), \Pr(Y, A_{1:K}))$, which measures the total variation distance between the learned and true joint distributions of $Y$ and $A_{1:K}$. The above errors capture how closely the learned models approximate the underlying true distributions.

\begin{theorem}[Compositional Error]
    \label{thm:main-theorem}
    Under Assumptions \ref{assumption1} and \ref{assumption2}, the error of an \methodname~is bounded by a linear combination of the errors of the attribute recognition model and the circuit-based task predictor. In particular, the error of the attribute recognition model across all attributes is bounded by the sum of the errors for each attribute, i.e.,
    \begin{equation*}
        \epsilon_{\theta, w} \leqslant \epsilon_{\theta} + 2\epsilon_w \leqslant \sum_{k=1}^K \epsilon_{\theta}^k + 2\epsilon_w.
    \end{equation*}
\end{theorem}

\begin{proof}[Proof Sketch]
    By leveraging Equation (\ref{eq:inference}), the upper bound of $\epsilon_{\theta, w}$ is decomposed into two terms. The first term depends only on $\theta$ and represents the error of the attribute recognition model, while the second term depends only on $w$ and captures the error of the circuit. In particular, the overall error of the attribute recognition model is further expanded into the errors across individual attributes.
\end{proof}
The complete proof is deferred to Appendix \ref{sec:omit}.
Theorem \ref{thm:main-theorem} demonstrates that reducing the error for any single attribute helps reduce the overall error of the attribute recognition model. More importantly, the error bound of an \methodname~is decomposable into contributions from its individual modules, which accredits to the compositional nature of \methodname s and the incorporation of probabilistic circuits. Consequently, reducing the error of any individual module helps improve the performance of the \methodname.

\section{Model Explanations}
As discussed in Section \ref{sec:model_architecture}, model predictions can be interpreted using the attribute recognition results and the conditional dependencies between classes and attributes. To further enhance the human's understanding of the model predictions, we provide various explanations addressing the following questions:
\textit{1) Which assignment of attributes contributes most to the model prediction?}
\textit{2) In cases where the model prediction was incorrect, could an adjustment to the attribute recognition results lead to a correct prediction?}
With a slight abuse of notation, let $\theta$ denote the parameters of the trained attribute recognition model and let $S, w$ denote the structure and parameters of the constructed probabilistic circuit.

\subsection{Most Probable Explanations}
To address the first question, we define Most Probable Explanations (MPEs) for \methodname s for identifying the highest contributing attribute assignments.

\begin{definition}[Most Probable Explanations]
    Given an input $x$ with prediction $\hat{y}$, the most probable explanation is defined as the attribute assignment that contributes the most in $\Pr_{\theta, w}(Y=\hat{y} \mid X=x)$. Formally,
    $a_{1:K}^* := \arg\max_{a_{1:K}}\operatorname{Pr}_{w}\left(Y=\hat{y} \mid A_{1:K} = a_{1:K}\right) \cdot \prod_{k=1}^K \operatorname{Pr}_{\theta}\left(A_k=a_k \mid X=x\right)$.
\end{definition}
MPE inference is generally challenging for probabilistic circuits. Such inference is tractable for selective circuits~\citep{sanchez2021sum}, but this type of circuit is relatively restrictive in expressiveness. As the number of attributes is small in our experimental settings, we simply use the brute-force algorithm to infer MPEs. Developing more efficient heuristics for MPE inference remains an open problem and is not the primary focus of this paper. Therefore, we leave it for future work.

MPEs provide a concrete explanation as to how the model arrives at a specific class prediction. Specifically, the predicted class is primarily due to the input image's attributes being recognized as $a_{1:K}^*$. These explanations offer attribute-level insights into the model's predictions, thereby enhancing interpretability and the human's understanding of the predictions.

To gain deeper insights into how these explanations represent the model's behavior, we define a property for MPEs, called \textit{alignment}, and introduce a corresponding metric to characterize the behavior of the model.
\begin{definition}[Alignment Rate]
    An MPE $a_{1:K}^*$ is considered aligned with a sample $x$ if $a_{1:K}^*$ matches the ground-truth attribute assignments of $x$, \ie, $a_k^* \in \{ j\in [q_k] \mid g_k^j(x)>0 \},~k\in [K]$.
    The alignment rate is defined as the proportion of aligned MPEs among all correctly predicted samples.
\end{definition}

A high alignment rate reflects strong model reliability, as it suggests that the ground-truth attribute assignments contribute the most during prediction. In other words, the model closely adheres to the human's understanding when making predictions.




\subsection{Counterfactual Explanations}
To address the second question, we define Counterfactual Explanations (CEs)~\citep{wachter2017counterfactual} for \methodname s to explore admissible changes in attribute recognition results that can correct any incorrectly predicted classes.

\begin{definition}[Counterfactual Explanations]
    Given an input $x$ which has an incorrect model prediction.
    Denote $\mathbf{b} := \{\mathbf{b_k}\}_{k\in [K]}$, $\mathbf{b_k} := \left(\ldots~b_{k, a_k}~\ldots\right)_{a_k\in\mathcal{A}_k}^{\top}$, and $\Pr_{\mathbf{b}}(Y=y \mid X=x) := \sum_{a_{1:K}} \Pr_w\left(Y=y \mid A_{1:K}=a_{1:K}\right) \cdot \prod_{k=1}^K b_{k, a_k}$.
    The counterfactual explanation for the ground-truth label $y$ is a set of probability vectors $\mathbf{b}$ that maximizes $\Pr_{\mathbf{b}}(Y=y \mid X=x)$, \ie, the solution to the following optimization problem,
    {\small
    \begin{equation}
        \begin{aligned}
            \max_{\mathbf{b}} \operatorname{Pr}_{\mathbf{b}}(Y=y \mid X=x), ~\text{s.t.} ~\sum_{a_k\in\mathcal{A}_k} b_{k, a_k} = 1~(0 \leqslant b_{k, a_k} \leqslant 1),~\forall k \in [K].
        \end{aligned}
        \label{optiz_problem}
    \end{equation}
    }
\end{definition}
We adopt the projected gradient ascent algorithm to generate the CEs, which is detailed in Algorithm \ref{alg:ce_optimization}.

\begin{minipage}{5.5in}
\begin{algorithm}[H]
\caption{Generation of Counterfactual Explanations}
\label{alg:ce_optimization}
\begin{algorithmic}[1]
\State \textbf{Input:} Feasible region $\mathcal{C} := \{\mathbf{b}: \sum_{a_k \in \mathcal{A}_k} b_{k,a_k} = 1~(0 \leqslant b_{k,a_k} \leqslant 1),~\forall k \in [K]\}$, initial values $b_{k,a_k}^{(0)} := \Pr_{\theta}(A_k=a_k \mid X=x)$,~$a_k \in \mathcal{A}_k$,~$k \in [K]$, learning rate $\gamma$

\For{$t = 0, 1, \ldots, T-1$}
    \Statex \quad~\# Computing $b_{k,a_k}^{(t+1)} := \mathcal{P}_{\mathcal{C}}\left\{b_{k,a_k}^{(t)}+\gamma \frac{\partial \log \Pr_{\mathbf{b}^{(t)}}(Y=y \mid X=x)}{\partial b_{k,a_k}}\right\}$~\citep{wang2013projection}
    \For{$k= 1, \ldots, K$}
        \State Sort $\{ b_{k,a_k}^{(t)}+\gamma \frac{\partial \log \Pr_{\mathbf{b}^{(t)}}(Y=y \mid X=x)}{\partial b_{k,a_k}} \}_{a_k\in\mathcal{A}_k}$ into $u_{k,1} \geqslant u_{k,2} \geqslant \ldots \geqslant u_{k,q_k}$
        \State Find $\rho_k := \sum_{j=1}^{q_k} \mathbb{I}\left( u_{k,j} > \frac{1}{j} \left( \sum_{l=1}^j u_{k,l} - 1 \right) \right)$
        \State Define $\lambda_k := \frac{1}{\rho_k} \left( 1 - \sum_{j=1}^{\rho_k} u_{k,j} \right)$
    \EndFor
    \State Let $b_{k,a_k}^{(t+1)} := \max\left\{b_{k,a_k}^{(t)} + \gamma \frac{\partial \log \Pr_{\mathbf{b}^{(t)}}(Y=y \mid X=x)}{\partial b_{k,a_k}} + \lambda_k,~0\right\}$,~$a_k\in\mathcal{A}_k$,~$k\in[K]$
\EndFor
\State \textbf{Output:} Counterfactual explanations $\mathbf{b}^{(T)}$
\end{algorithmic}
\end{algorithm}
\end{minipage}

CEs reveal the model's inner workings by identifying the attribute recognition changes required to correct an incorrect prediction. Similar to MPEs, these explanations deliver attribute-level insights into the model's decision-making process, thereby improving interpretability.

Next, we introduce a metric to evaluate the effectiveness of CEs in correcting the model predictions.
\begin{definition}[Correction Rate]
    The correction rate is defined as the proportion of initially incorrectly predicted samples that are corrected by CEs among all initially incorrectly predicted samples.
\end{definition}
A high correction rate signifies that the generated CEs effectively correct model predictions by adjusting the attribute recognition results.




\section{Experiments} \label{sec:exp}
\subsection{Experimental Setup} \label{sec:setup}
\paragraph{Datasets}
We evaluate the model performance on a variety of benchmark datasets.
\textbf{1) MNIST-Addition:} We derive this dataset from the original MNIST dataset~\citep{mnist} by following the general preprocessing steps and procedures detailed in~\citep{DeepProbLog}. Each MNIST-Addition sample consists of two images randomly selected from the original MNIST. The digits in these images, ranging from 0 to 9, represent two attributes, with their sum serving as the class label. A total of 35,000 samples are created for MNIST-Addition.
\textbf{2) GTSRB:} GTSRB~\citep{gtsrb} is a dataset comprising 39,209 images of German traffic signs, with class labels indicating the type of signs. Additionally, we annotate each sample with four attributes: ``color'', ``shape'', ``symbol'', and ``text''. The values for these attributes are detailed in Appendix \ref{sec:configuration}.
\textbf{3) CelebA:} CelebA~\citep{celeba} consists of 202,599 celebrity face images annotated with 40 binary concepts. Here, we select the 8 most balanced binary concepts\footnote{Certain concepts are excluded due to political or ethical concerns.} and group them into 5 attributes: ``mouth'', ``face'', ``cosmetic'', ``hair'', and ``appearance''. Following \citet{cem}, each unique combination of concept values is treated as a group. To balance the dataset and increase its complexity, we rank these groups by the number of images they contain and pair them strategically: the group with the most images is merged with the one with the fewest, the second most with the second fewest, and so on. The above strategy results in 127 total classes.
\textbf{4) AwA2:} AwA2~\citep{awa2} contains 37,322 images of 50 types of animals, each annotated with 85 binary concepts. Certain concepts, such as those describing non-visual properties (\eg, ``fast'', ``domestic'') or indistinctive features (\eg, ``chewteeth''), or those representing background information (\eg, ``desert'' and ``forest''), are excluded. After the exclusion, 29 concepts remain, which are then grouped into 4 attributes: ``color'', ``surface'', ``body'', and ``limb''. The specific values for these attributes are provided in Appendix \ref{sec:configuration}.
For all datasets, we split the samples into training, validation, and testing sets by a ratio of 8:1:1.

\paragraph{Baselines}
We select CBM~\citep{cbm} and several representative variants as baselines. Specifically, we choose CEM~\citep{cem}, a method that uses high-dimensional concept embeddings as the bottleneck instead of concept probabilities, and DCR~\citep{dcr}, which introduces a deep concept reasoner as the task predictor rather than relying on a simple linear layer. Additionally, we train an end-to-end DNN~\citep{resnet} as an additional baseline. It is important to note that CEM and the end-to-end DNN are not interpretable, as their components are not explicitly understandable by humans, even though they may achieve competitive performance on downstream tasks. A comparison of model properties is summarized in Table \ref{tab:property}.
Detailed descriptions of model architectures and training details are deferred to Appendix \ref{sec:more_setup}.

\begin{table}[tb]
    \centering
    \caption{ Model properties. ``Interpretability'' indicates whether the outputs produced by a concept/attribute recognition model are interpretable and whether humans can interpret the final decisions using these outputs. ``Data-Driven Rules'' denotes whether a model can incorporate logical rules learned from data. ``Human-Predefined Rules'' specifies whether a model can integrate logical rules predefined by humans. ``Theoretical Guarantee'' indicates whether a model provides a theoretical guarantee on the relationship between the performance of the overall model and that of its individual components.
    }
    \label{tab:property}
    \resizebox{5.7in}{!}{%
    \begin{tabular}{lcccc}
        \toprule
        \multirow{ 2}{*}{\textbf{Model}} & \multirow{ 2}{*}{\textbf{Interpretability}} & \textbf{Data-Driven} & \textbf{Human-Predefined} & \textbf{Theoretical} \\
        & & \textbf{Rules} & \textbf{Rules} & \textbf{Guarantee} \\
        \midrule
        CBM~\citep{cbm}         & \tick & \cross & \cross & \cross \\
        CEM~\citep{cem}         & \cross & \cross & \cross & \cross \\
        DCR~\citep{dcr}         & \tick & \tick & \cross & \cross \\
        \methodname~(ours)      & \tick & \tick & \tick & \tick \\
        \bottomrule
    \end{tabular}}
\end{table}

\paragraph{Evaluation Metrics}
Considering the compositional nature of \methodname s, we introduce distinct evaluation metrics for the various modules as well as the overall model.
\textbf{1) Attribute Recognition Model:} 
We employ two metrics for evaluating the attribute recognition model. First, we propose the \textit{mean total variation (TV) distance} between the output probability vectors from the attribute recognition model and the corresponding ground-truth probability vectors. The metric is defined as: $\frac{1}{K}\sum_k\frac{1}{|D_{\text{test}}|}\sum_{x \in D_{\text{test}}} d_{\text{TV}}(f_k(x;\theta), g_k(x))$, where a smaller distance indicates better performance. For baseline models producing concept probabilities, we group the outputs into attributes and process them as probability vectors. 
Second, we adapt the \textit{mean concept accuracy} metric used in~\citep{cem} for \methodname s. For an input $x$ with $n_k$ ground-truth values for the attribute $A_k$, we identify the top-$n_k$ values in $f_k(x; \theta)$ and set them to one while the rest are set to zero. We then calculate the accuracy by comparing these processed outputs with the ground-truth concept labels.
\textbf{2) Circuit-Based Task Predictor:} 
We use the \textit{mean likelihood} to measure the performance of circuit-based task predictors, which is given by: $\frac{1}{|\bar{D}_{\text{test}}|}\sum_{(y, a_{1:K}) \in \bar{D}_{\text{test}}} \Pr_w(Y=y, A_{1:K}=a_{1:K})$. A larger mean likelihood suggests that the circuit models the joint distribution more accurately.
\textbf{3) Overall Model:} 
We adopt the \textit{standard classification accuracy} as the evaluation metric for the overall model.

\subsection{\methodname s vs. Baselines}
We compare \methodname s against baseline models across the four benchmark datasets, with the results summarized in Table \ref{tab:performance_comparison}. Specifically, we refer to the \methodname~whose circuit was learned using the data-driven approach as ``\methodname (Data)'' and the \methodname~whose circuit was constructed using the knowledge-injected approach as ``\methodname (Knowledge)''.

\begin{table}[tb]
  \centering
  \caption{Classification accuracy of \methodname s and four baseline models on four benchmark datasets over five random seeds. ``*'' denotes uninterpretable models. The best results for each dataset are highlighted in bold, while the second-best results are underlined.}
  \resizebox{5.7in}{!}{%
  \begin{tabular}{lcccc}
    \toprule
    Model & MNIST-Add (\%) & GTSRB (\%) & CelebA (\%) & AwA2 (\%) \\
    \midrule
    $\text{DNN}^*$ & $99.057 \pm 0.08$ & $\underline{99.939} \pm 0.04$ & $\mathbf{36.963} \pm 0.72$ & $\mathbf{93.351} \pm 0.17$ \\
    CBM        & $98.606 \pm 0.03$ & $99.810 \pm 0.04$ & $16.552 \pm 0.87$ & $82.286 \pm 0.47$ \\
    $\text{CEM}^*$        & $98.740 \pm 0.10$ & $99.736 \pm 0.06$ & $25.218 \pm 0.30$ & $85.102 \pm 0.27$ \\
    DCR        & $94.597 \pm 2.05$ & $87.071 \pm 6.93$ & $7.055 \pm 3.04$  & $44.117 \pm 10.05$ \\
    \methodname (Data) & $\underline{99.171} \pm 0.11$ & $99.888 \pm 0.08$ & $\underline{33.739} \pm 0.90$ & $\underline{87.281} \pm 0.39$ \\
    \methodname (Knowledge)   & $\mathbf{99.189} \pm 0.08$ & $\mathbf{99.944} \pm 0.04$ & $31.727 \pm 0.51$ & $68.519 \pm 3.54$ \\
    \bottomrule
  \end{tabular}}
  \label{tab:performance_comparison}
\end{table}

The results in Table \ref{tab:performance_comparison} demonstrate that NPCs outperform all other concept-based baseline models. Specifically, \methodname (Knowledge) achieves the best performance on the MNIST-Addition and GTSRB datasets, while \methodname (Data) leads on the CelebA and AwA2 datasets. Notably, \methodname s achieve superior performance even compared to CEM, an uninterpretable model that relies on high-dimensional concept embeddings, highlighting \methodname's effectiveness in leveraging interpretable concept probabilities for downstream classification tasks.

Remarkably, \methodname s achieve superior performance even compared to that of the end-to-end DNN, surpassing its classification accuracy on the MNIST-Addition and GTSRB datasets. Nevertheless, small gaps remain between the end-to-end DNN and \methodname  on more complex datasets like CelebA and AwA2. The above findings demonstrate that, while there is still small room for improvement compared to black-box models, NPCs strike a compelling balance between interpretability and task performance. Overall, the results underscore the remarkable potential of interpretable models, demonstrating their ability to achieve competitive performance in downstream tasks compared to baselines and even black-box end-to-end DNN models.

\subsection{Ablation Studies}
In this section, we delve into \methodname s from additional perspectives. Specifically, we will analyze the integration of attributes, the influence of attribute selections, the effects of various approaches to constructing the task predictor, and, lastly, the impact of joint optimization.

\subsubsection{Attributes vs. Concepts}
Unlike existing concept-based models that utilize individual binary concepts (\eg, ``red color'', ``yellow color''), \methodname s focus on concept groups, \ie, attributes (\eg, ``color''). Here, we aim to explore \textit{the benefits of using attributes compared to individual concepts}. To this end, we replace the concept recognition model in CBM~\citep{cbm} with an attribute recognition model, resulting in a new model called the Attribute Bottleneck Model (ABM). ABM comprises an attribute recognition model and a linear layer serving as the task predictor. We employ the training loss from CBM to train ABM, replacing the concept loss with the attribute loss as defined in Equation~(\ref{eq:attribute_loss}).
The performance comparison between CBM and ABM is presented in Table~\ref{tab:attribute_concept}.

The results in Table~\ref{tab:attribute_concept} show that, in terms of the mean TV distance, ABM outperforms CBM on the MNIST-Addition and GTSRB datasets, albeit exhibiting slightly worse performance on the CelebA and AwA2 datasets. On the other hand, ABM consistently surpasses CBM in mean concept accuracy across all datasets. These findings underscore the effectiveness of the attribute recognition model, indicating that attributes capture more nuanced information. In particular, we hypothesize that inherent relationships may exist both within the values of any single attribute and across different attributes. Therefore, treating all values as independent concepts neglects these interdependencies, potentially leading to inferior performance. Overall, these results suggest the benefits of utilizing attributes in preserving relational constraints within predictions and thus improving model performance.

\begin{table}[tb]
  \centering
  \caption{Comparison between CBM and ABM across four benchmark datasets in terms of mean TV distance and mean concept accuracy. For TV distance, lower values are better. For accuracy, higher values are better. The better values are in bold.}
  \resizebox{5in}{!}{%
  \begin{tabular}{clcccc}
    \toprule
    Metric & Model & MNIST-Add & GTSRB & CelebA & AwA2 \\
    \midrule
    \multirow{2}{*}{Mean TV Distance} 
      & CBM & 0.0080 & 0.0023 & \textbf{0.1744} & \textbf{0.0514} \\
      & ABM & \textbf{0.0058} & \textbf{0.0007} & 0.1759 & 0.0775 \\
    \midrule
    \multirow{2}{*}{Mean Concept Accuracy} 
      & CBM & 98.64\% & 99.80\% & 32.52\% & 82.21\% \\
      & ABM & \textbf{98.99\%} & \textbf{99.84\%} & \textbf{45.00\%} & \textbf{89.69\%} \\
    \bottomrule
  \end{tabular}}
  \label{tab:attribute_concept}
\end{table}

\subsubsection{Impact of Attribute Selection}
During inference, an \methodname~utilizes sufficient attributes to produce final predictions. Here, we aim to investigate the following questions: \textit{How does the exclusion of one particular attribute during inference impact the performance of \methodname~on downstream tasks?} \textit{How does the exclusion of different attributes vary the impact?}

Assuming the excluded attribute is $A_1$, the predicted score for the class $y$ would become $\sum_{a_{2:K}}\Pr_w(Y=y\mid A_{2:K}=a_{2:K})\cdot\prod_{k=2}^K\Pr_\theta(A_k=a_k\mid X)$. We apply this inference process on the GTSRB dataset and the resulting task performance is presented in Figure \ref{fig:reduce} (Left).


\begin{figure}[htbp]
\centering
\includegraphics[width=5.7in]{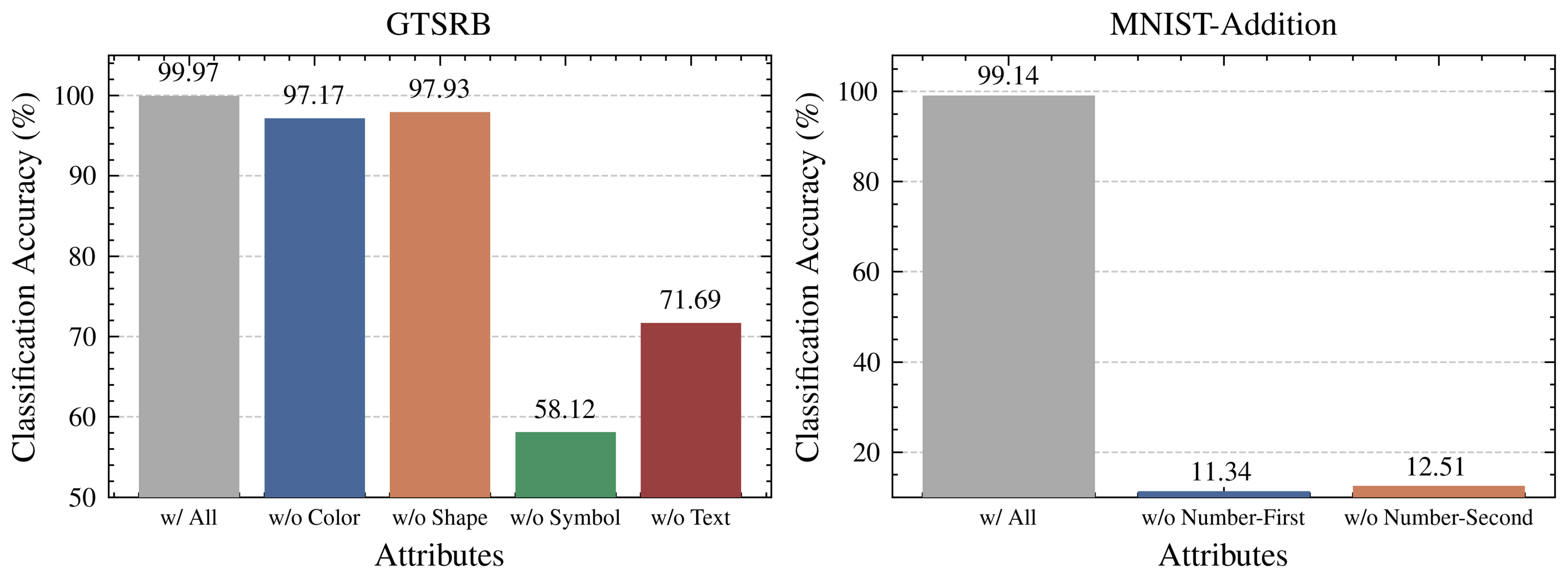}
\caption{\methodname~classification accuracy on the GTSRB and MNIST-Addition datasets with attribute exclusions. The gray bars indicate inference with all attributes, while non-gray bars indicate inference with one particular attribute excluded.}
\label{fig:reduce}
\end{figure}

We observe that, when \methodname~has one fewer attribute at its disposal during inference, the classification accuracy decreases, due to the fact that the assumption of sufficient attributes (Assumption \ref{assumption1}) no longer holds in this scenario, and that the equality in Equation (\ref{eq:inference}) was violated. Therefore, the formula above does not correctly represent $\Pr_{\theta,w}(Y=y\mid X)$. As a result, relying on the above formula for inference can adversely affect the predictions.

On the other hand, we observe that excluding different attributes can lead to varying impacts on the task performance. Specifically, excluding the ``color'' or ``shape'' attributes results in only a slight drop in accuracy, whereas excluding the ``symbol'' or ``text'' attributes leads to a significant decline. We accredit such discrepancy to the distinct properties of these attributes. 
More specifically, attributes such as ``color'' and ``shape'' are generally not decisive, meaning that they do not decisively determine the final class, and that their absence can be compensated by leveraging information from the other attributes. For instance, even without having any indication of ``red'' or ``octagon'', it is still possible to infer an input representing a stop sign so long as the ``text'' attribute indicates ``stop''.  Thus, excluding the non-decisive attributes has minimal impact on the performance.
In contrast, ``symbol'' and ``text'' attributes are decisive for many instances and thus are crucial for distinguishing between certain classes. For example, without the ``text'' attribute, it then becomes impossible to differentiate speed limit signs indicating different speeds. Similarly, without the ``symbol'' attribute, distinguishing between a left-curve and a right-curve sign is no longer feasible. Hence, excluding decisive attributes can severely impair predictions.

For the MNIST-Addition dataset, Figure \ref{fig:reduce} (Right) demonstrates similar findings. In particular, as both attributes are essential for determining the final class, \ie, the sum, the removal of either attribute results in a substantial performance decline.

In summary, utilizing insufficient attributes compromises \methodname's performance on downstream tasks, and the impact of excluding different attributes varies depending on the properties of the attributes themselves.

\subsubsection{Impact of Task Predictor Construction Approaches}
In Section \ref{sec:relation_learning}, we introduce two distinct approaches for constructing probabilistic circuits: the data-driven approach and the knowledge-injected approach. Here, our objective is to investigate the impact of these construction methods. Specifically, we aim to address the following questions: \textit{Which approach produces a circuit that better captures the data distribution? Which approach produces a circuit that performs more effectively as a task predictor?}
We start by comparing the mean likelihood of the two circuits. Then, we examine the classification accuracy of the overall models consisting of a well-trained attribute recognition model together with either circuit (data driven or knowledge injected). The results are summarized in Table \ref{tab:circuit}. 

\begin{table}[htbp]
  \centering
  \caption{Performance comparison between circuits constructed using data-driven and knowledge-injected approaches, along with the classification accuracy of overall models combining either circuit with a well-trained attribute recognition model. The better results are in bold.}
  \resizebox{5.7in}{!}{%
  \begin{tabular}{llcccc}
    \toprule
    \textbf{Metric} & \textbf{Architecture} & \textbf{MNIST-Add} & \textbf{GTSRB} & \textbf{CelebA} & \textbf{AwA2} \\
    \midrule
    \multirow{2}{*}{Mean Likelihood} 
    & Circuit (Data)      & \textbf{1.010e-2} & \textbf{3.663e-02} & 2.082e-02          & \textbf{1.263e-04} \\
    & Circuit (Knowledge) & 1.007e-2          & \textbf{3.663e-02} & \textbf{2.091e-02} & 1.255e-04          \\
    \midrule
    \multirow{2}{*}{Accuracy} 
    & Model (Data)        & \textbf{99.17\%}  & \textbf{99.85\%} & \textbf{32.74\%} & \textbf{68.52\%} \\
    & Model (Knowledge)   & \textbf{99.17\%}  & \textbf{99.85\%} & 31.56\%          & 23.47\%          \\
    \bottomrule
  \end{tabular}}
  \label{tab:circuit}
\end{table}

In terms of classification accuracy, the models combining the different circuits exhibit similar performance on MNIST-Addition, GTSRB, and CelebA datasets. Such similarity suggests that the two-layered circuit constructed using the knowledge-injected approach is sufficient to provide accurate relational information among attributes and classes for simpler datasets.

In contrast, for the more complex AwA2 dataset, Model (Data) outperforms Model (Knowledge) by a wide margin due to the presence of multi-valued attributes in AwA2, which results in a large number of valid combinations of attribute values. Consequently, each combination may correspond to a very small joint probability. Under these circumstances, even a slight difference in mean likelihood may induce a significant change in the circuit's ability to capture the data distribution. For instance, even a marginally lower mean likelihood might indicate the failure of the circuit to capture the joint probabilities of a large quantity of combinations.
As shown in Table \ref{tab:circuit}, the mean likelihood of Circuit (Knowledge) is slightly lower than that of Circuit (Data) for the AwA2 dataset, suggesting that, in this case, the knowledge-injected circuit might have failed to capture sufficient nuanced relationships among attributes and classes, while the data-driven circuit is better suited for representing the joint distribution of the data, thus leading to a better performance on the downstream task.

\subsubsection{Impact of Joint Optimization}
When training the \methodname s, we adopt a three-stage training algorithm, where we first independently train the attribute recognition model and the task predictor, followed by a joint optimization for the overall model. Here, we aim to investigate \textit{how the third stage, \ie, joint optimization, affects the performance of \methodname s}. To this end, we compare the performance of \methodname s before and after applying the joint optimization. The comparison is illustrated in Figure \ref{fig:joint}.

\begin{figure}[htbp]
\centering
\includegraphics[width=5.7in]{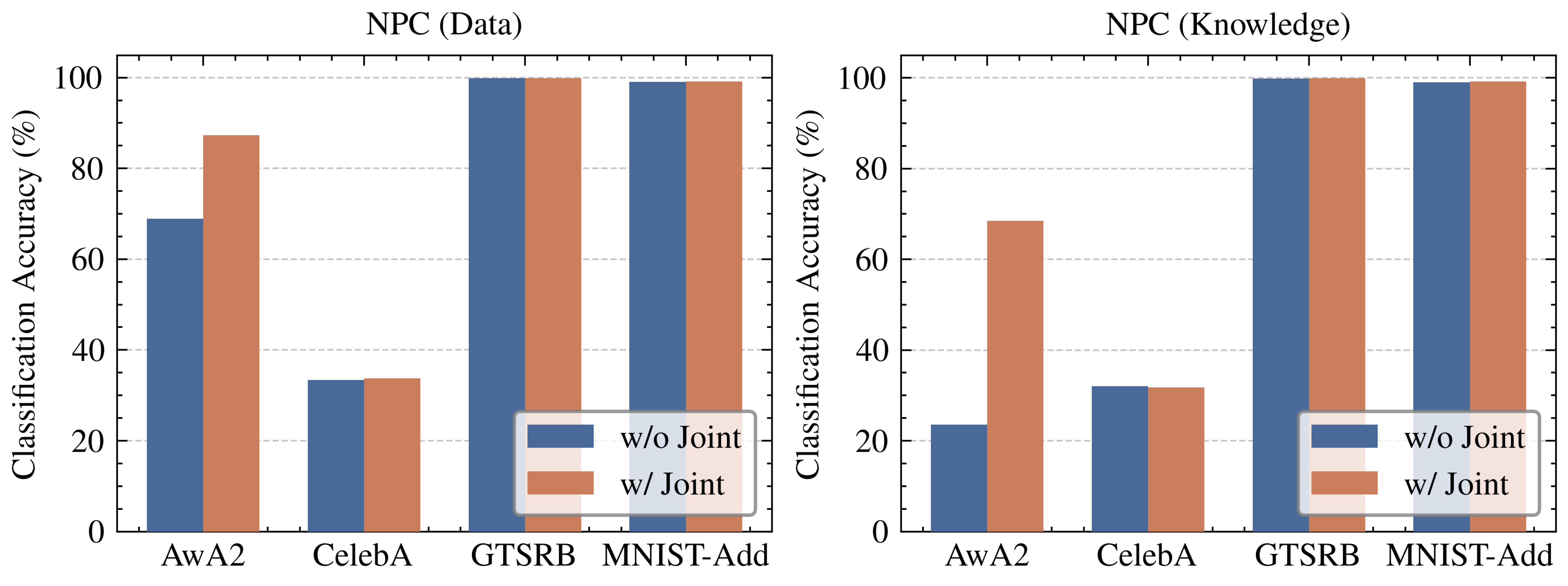}
\caption{Classification accuracy of \methodname (Data) and \methodname (Knowledge) on the four benchmark datasets. The blue bars indicate performance prior to joint optimization, while the orange bars illustrate performance after applying joint optimization.}
\label{fig:joint}
\end{figure}

In general, joint optimization enhances model performance across the various datasets. Specifically, for the AwA2 dataset, joint optimization leads to substantial improvements for both \methodname (Data) and \methodname (Knowledge), demonstrating its effectiveness. In contrast, for the CelebA dataset, the performance stays roughly the same after applying joint optimization, with a marginal increase in performance for \methodname (Data) and a marginal decrease in performance only for
\methodname (Knowledge). Nevertheless, slight performance improvements are generally observed on the GTSRB and MNIST-Addition datasets for both \methodname s. Overall, these results suggest that joint optimization provides additional benefits for \methodname s in terms of downstream tasks. In particular, for datasets with superior performance after the initial training, joint optimization provides additional but modest gains. In contrast, for datasets with moderate performance after the initial training, joint optimization plays a pivotal role in delivering significant performance improvements.

\subsection{Model Explanations}
In this section, we explore two types of explanations and provide examples to illustrate how the explanations facilitate the human's understanding of \methodname 's inner workings and interpret the model's behavior.

\subsubsection{Most Probable Explanations}
Figure \ref{fig:mpe_learn} presents some examples for \methodname (Data) from the four benchmark datasets. Specifically, each example comprises an image, the ground-truth labels for the class and attributes, the class predicted by \methodname (Data), and, lastly, the corresponding MPE which accounts for the attribute assignment that contributes most significantly to the prediction. In these instances, \methodname (Data) provides correct class predictions, and the MPEs are aligned with the ground-truth attribute labels. For instance, the MPE for the example from GTSRB is \{Color: Red; Shape: Circle; Symbol: Text; Text: 30\}, which precisely matches the ground-truth attribute labels. Such alignment between MPEs and attribute labels indicates that the model employs human-like reasoning and makes reliable decisions. Examples for \methodname (Knowledge) are deferred to Appendix \ref{sec:more_explanation}.

\begin{figure}[htbp]
    \centering
    \includegraphics[width=5.7in]{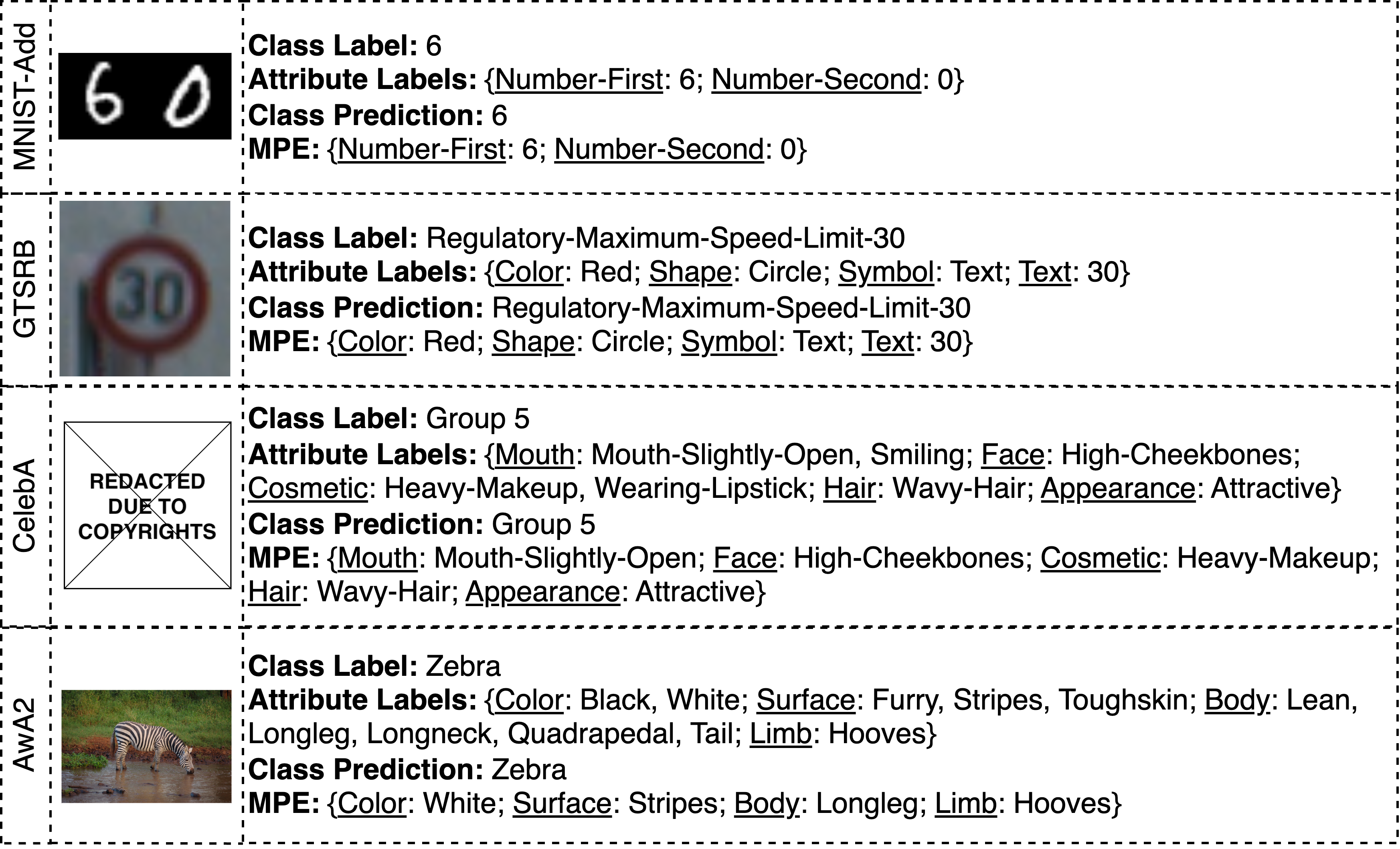}
    \caption{Examples from the four benchmark datasets. Each example includes an image, the ground-truth class label, the ground-truth attribute labels, the \methodname (Data) class prediction, and, lastly, the corresponding MPE. These examples illustrate MPEs that align with the ground-truth attribute labels. The CelebA image is redacted in compliance with its terms of use.}
    \label{fig:mpe_learn}
\end{figure}

The MPEs alignment rates are shown in Figure \ref{fig:explanation} (Left). We observe that \methodname (Knowledge) exhibits a relatively low alignment rate on the AwA2 dataset, suggesting that the ground-truth attribute assignment does not contribute the most to the correct predictions for certain instances. Such misalignment could be caused by the knowledge-injected approach failing to capture the relatively more complex conditional dependencies between classes and attributes for this particular dataset, thereby affecting the model's prediction process. On the other hand, the MPEs alignment rates in the other scenarios are close to 100\%, indicating that when it makes correct predictions, the model predominantly relies on attribute assignments matching the ground truths. Therefore, the model is considered reliable as its prediction process properly aligns with the human's decision-making process.

\begin{figure}[tb]
\centering
\includegraphics[width=5.7in]{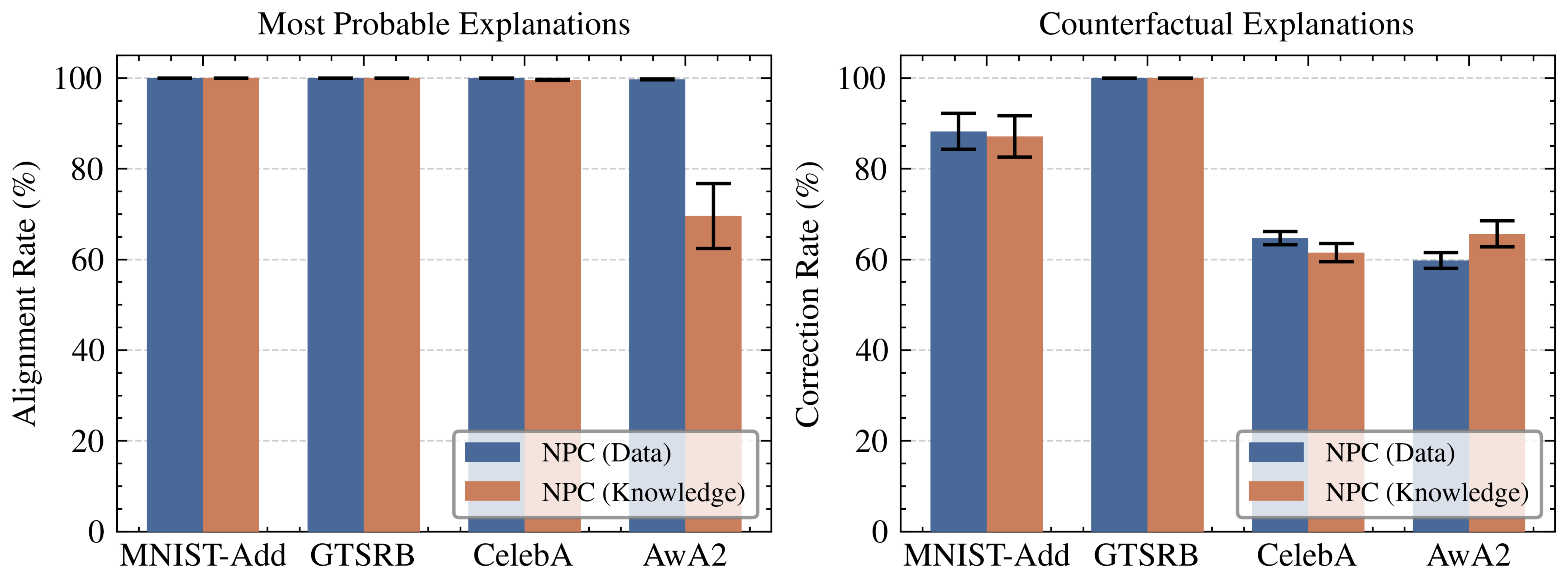}
\caption{MPE alignment rate and CE correction rate for \methodname s across the four benchmark datasets. The blue bars represent \methodname (Data), while the orange bars represent \methodname (Knowledge). Error bars indicate the standard deviation over five random seeds.}
\label{fig:explanation}
\end{figure}

\subsubsection{Counterfactual Explanations}
Figure \ref{fig:ce_learn} illustrates some examples for \methodname (Data) from the four benchmark datasets. Each sample consists of an image, the attribute and class predictions, the generated CE, and, lastly, the class prediction corrected by the CE.
In these examples, \methodname (Data) incorrectly predicts the class, while CEs effectively correct them by making minimal adjustments to the attribute predictions. For instance, for the MNIST-Addition dataset, changing the ``Number-Second'' attribute from ``4'' to ``9'' corrects the class prediction from ``11'' to ``16''. Similarly, for the GTSRB dataset, the predicted class is corrected by adjusting the ``Text'' attribute prediction. For the CelebA dataset, the CE yields the correct class prediction primarily by increasing the probability of ``Attractive'' from the ``Appearance'' attribute. Finally, for the AwA2 dataset, the CE corrects the class prediction from ``Horse'' to ``Deer'' primarily by increasing the probability of ``Brown'' for the ``Color'' attribute. Examples for \methodname (Knowledge) are provided in Appendix~\ref{sec:more_explanation}.

\begin{figure}[htbp]
    \centering
    \includegraphics[width=5.7in]{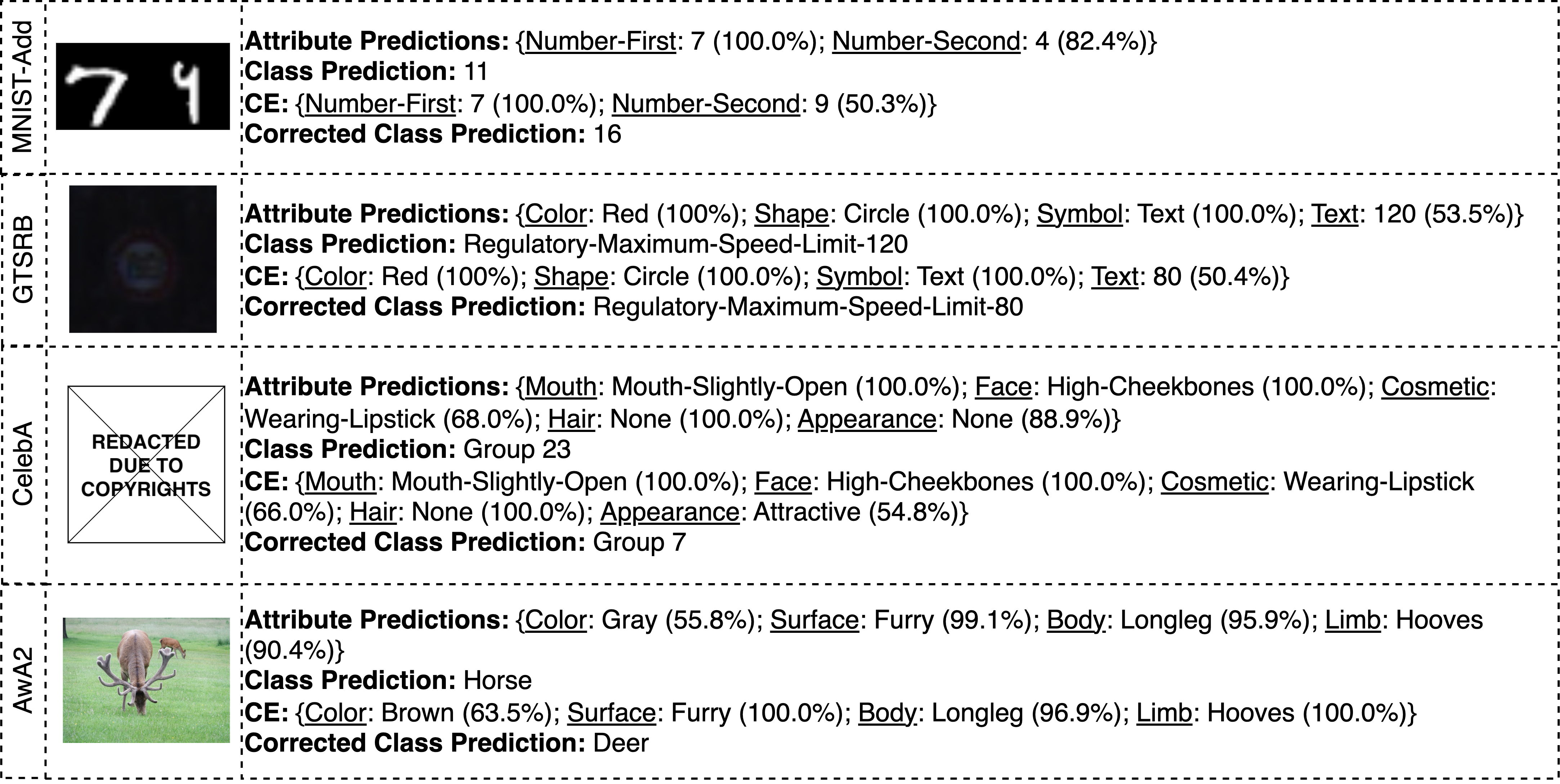}
    \caption{Examples from the four benchmark datasets. Each example includes an image, the \methodname (Data) attribute and class predictions, the corresponding CE, and the class prediction corrected by CE. These examples focus on CEs that successfully correct the class predictions. For simplicity, both attribute predictions and CEs display only the most confident value for each attribute. The CelebA image is redacted in compliance with its terms of use.}
    \label{fig:ce_learn}
\end{figure}

The CE correction rates are presented in Figure \ref{fig:explanation} (Right). For simpler datasets like MNIST-Addition and GTSRB, the generated CEs exhibit high correction rates, indicating their effectiveness in correcting predictions. However, for more complex datasets such as CelebA and AwA2, the correction rates are lower, highlighting the limit of the efficacy of the generated CEs on complex datasets. Such a limit underscores the future need for more advanced algorithms for generating CEs that can correct misclassifications with excellent correction rates, even for complex datasets.

\section{Limitations and Discussions} \label{sec:discuss}
In this section, we discuss the limitations of \methodname s from multiple perspectives, highlighting potential future directions for improvement.
\paragraph{Model Architecture}
Compared to end-to-end DNNs, \methodname s offer superior interpretability by decomposing a model into semantically meaningful modules, enabling humans to combine module outputs to understand the final decisions. Nevertheless, the attribute recognition model itself remains a black box, and its opaque inner workings make it difficult to ensure that its outputs truly represent the probabilities for the various attributes. For instance, the model might learn spurious correlations and incorrectly map background features, instead of actual attributes, to outputs. Future work may focus on increasing the transparency within the attribute recognition model, thereby enhancing its interpretability.

\paragraph{Structure of Probabilistic Circuits}
In \methodname s, the task predictor, implemented using a probabilistic circuit, is either learned using LearnSPN~\citep{learnspn} or manually constructed based on human-predefined rules. The circuit generated by LearnSPN, however, may contain an excessive number of nodes and edges, resulting in a slower inference. Alternative methods~\citep{VergariME15, alternative_spn} may be explored to create more compact circuits for added inference efficiency. On the other hand, manually constructed circuits employ simpler structures with only two layers. While it may improve efficiency, such simplicity may limit the circuit’s expressiveness, potentially degrading its performance on complex datasets like AwA2. Future work may focus on improved balancing between circuit expressiveness and structural complexity.

\paragraph{Assumption of Complete Information}
The assumption of complete information in this paper assumes that attributes are conditionally independent given the input. However, recent studies suggest that this assumption limits model expressiveness~\citep{condition_assump} and potentially introduces spurious correlations, also known as reasoning shortcuts~\citep{reasoning_shortcut1, reasoning_shortcut2}. Investigating whether this issue arises in \methodname s and identifying ways to address it would be a valuable direction for future research.


\paragraph{Reducing Trade-Offs between Interpretability and Task Performance}
In this paper, we show that, with the integration of attribute recognition and probabilistic circuit, \methodname~produces interpretable predictions for downstream tasks while achieving superior performance. Looking ahead, we believe that, by incorporating more fine-grained and diverse attributes that are semantically meaningful, along with a structure that reasons over these attributes using logical rules with increased complexity, we shall devise compositional model designs that further reduce the trade-offs between interpretability and performance of downstream tasks.




\section{Related Work} \label{sec:related}
In this section, we discuss several areas of research relevant to our
proposed method.

\paragraph{Concept Bottleneck Models and Variants}
Concept bottleneck models (CBMs) and their variants are a class of machine learning models that organize their decision-making process around high-level, human-understandable concepts, offering enhanced transparency. First introduced by~\citet{cbm}, CBMs decompose a black-box DNN into two modules: a concept recognition model, responsible for predicting various human-specified concepts, and a task predictor, which performs classifications on the predicted concepts.

Subsequent research has focused on improving these two modules. \citet{cem, concept_emb1, concept_emb2} extend the concept recognition model by representing concepts as high-dimensional embeddings rather than simple probabilities. Additionally, \citet{hybrid_cbm, concept_unsupervised1, concept_unsupervised2, concept_unsupervised3} introduce unsupervised neurons into the bottleneck to enhance the model's learning capacity. While improving the performance on downstream tasks, these extensions compromise interpretability, as the dimensions within concept embeddings and the unsupervised neurons lack explicit semantic meanings. In contrast, utilizing predicted concept probabilities gives better interpretability.
On the other hand, there have been recent efforts to improve the interpretability of the task predictor. Instead of using a linear layer as the task predictor, several approaches~\citep{dcr, len, soft_tree_cbm} design new architectures to embed logical rules and enable classifications via reasoning. For instance, \citet{dcr} propose a deep concept reasoner, while \citet{soft_tree_cbm} introduce a soft decision tree as the task predictor. These approaches optimize their parameters using observed data, thus extracting the underlying logical rules inherent within the data.
In comparison, architectures that directly encode human-predefined logical rules through their structure and parameters offer means to explicitly represent domain knowledge.

\paragraph{Probabilistic Circuits}
Probabilistic circuits~\citep{spn_survey} are rooted directed acyclic graphs designed to represent the joint distribution of a set of variables. The circuits comprise three types of nodes: 1) leaf nodes, which correspond to input variables; 2) sum nodes, which compute weighted sums of their child nodes; and 3) product nodes, which compute products of their child nodes. When satisfying the properties of decomposability and smoothness, a probabilistic circuit becomes a tractable probabilistic model, ensuring efficient inferences over various distributions~\citep{spn}. Specifically, joint, marginal, and conditional probabilities of input variables can be computed in at most two passes (from leaf nodes to the root node), with computational complexity linear in the size of the circuit. Consequently, probabilistic circuits combine the expressiveness of traditional graphical models with the scalability of modern deep learning frameworks.

Structure learning for probabilistic circuits aims to design structures that effectively balance expressiveness and computational efficiency. ~\citet{spn_structure_survey} categorize existing structure learning methods into four types: 1) handcrafted structure learning, where structures are manually designed for specific datasets~\citep{discriminative_spn, spn}; 2) data-based structure learning, which uses heuristic~\citep{AdelBG15, dennis2012learning, learnspn, krakovna2016minimalistic, mix_spn, merge_spn, RooshenasL14, VergariME15} or non-heuristic algorithms~\citep{selective, lee2014non, trapp2016structure, PeharzVS0TSKG19} to learn structures from data; 3) random structure learning, where structures are randomly generated as a flexible starting point~\citep{PeharzVS0TSKG19, RashwanZP16, bayesian_spn}; and 4) ensemble structure learning, which combines multiple structures to improve generalization to high-dimensional data~\citep{ventola2020residual}. In this paper, we utilize the first and second types of structure learning approaches to embed explicit and implicit logical rules, respectively.

Parameter learning for probabilistic circuits involves finding optimal parameters for a given structure, enabling the circuits to accurately capture the underlying probability distributions within the observed data. Parameter learning can be broadly categorized into two types: generative and discriminative. Generative parameter learning~\citep{spn, peharz2015foundations, RashwanZP16, ZhaoAGA16, cccp}, the most common paradigm, aims to maximize the joint probabilities of all variables. The generative approach is particularly suited for tasks such as density estimation, generative modeling, and probabilistic reasoning. In contrast, discriminative parameter learning~\citep{discriminative_spn, AdelBG15, discriminative_extend_spn} focuses on maximizing the conditional probabilities of a class variable given other variables, making it ideal for classification and regression tasks. In this paper, we adopt CCCP~\citep{cccp}, a generative parameter learning approach, as it admits multiplicative parameter updates that provide a monotonic increase in the log-likelihood and lead to faster and more stable convergence.

On the other hand, some works develop probabilistic circuits parameterized by neural networks.
For instance, \citet{semantic_layer} use an amortized neural network to output the weights in circuits.
To address the overfitting problem of large probabilistic circuits, \citet{hyper_spn} exploit the generalization ability of a small-scale neural network and employ it to generate the parameters of a large circuit.
\citet{conditional_pc} aims to learn a probabilistic circuit to model the conditional distribution of target variables given input variables. In their approach, a neural network, conditioned on the input variables, is utilized to generate the circuit's parameters. This integration of neural networks enhances the expressiveness of probabilistic circuits and has been used in neuro-symbolic integration~\citep{semantic_layer, DeepProbLog}.
In this paper, the circuit's integration with the attribute recognition model can be seen as parameterizing the input distributions of the circuit, instead of the weights of the circuit. The final predictions are a result of utilizing the outputs of a circuit through the law of total probability.

\paragraph{Integration of Probabilistic Graphical Models}
Probabilistic graphical models (PGMs) are frameworks that use graphs to express conditional dependencies between variables and represent their joint probability distributions. With their expressiveness, PGMs can be integrated into the decision-making process to enhance models from various perspectives. Our work demonstrates one approach to integrating PGMs, \ie, probabilistic circuits, to improve the transparency and interpretability of a model's predictive process. In contrast, \citet{reason_yang, reason_kemlp, reason_care, reason_colep} have focused on leveraging PGMs integration to enhance the adversarial robustness of deep classification models.

\paragraph{Neuro-Symbolic Learning}
Neuro-symbolic learning integrates neural networks with symbolic representations, combining data-driven learning with symbolic reasoning to leverage the strengths of both. Variants of CBMs that embed inherent rules within the task predictor serve has one exemplifying application of neuro-symbolic learning. Beyond CBMs, this neuro-symbolic paradigm can be implemented in various other forms.
One line of research focuses on designing symbolic-based objective functions. For instance, \citet{ltn} propose objectives that maximize the satisfiability of predefined symbolic rules over a neural network's outputs. Similarly, \citet{semantic_loss, semantic_loss2} define objectives that maximize the probabilities of generating outputs aligned with symbolic rules. These objectives can also function as regularization terms alongside standard classification losses, encouraging a neural network to adhere to specific rules by optimizing its parameters accordingly.
Another line of research emphasizes the design of model architectures. For example, \citet{semantic_layer} introduce a semantic probabilistic layer, a predictive layer designed for structured-output predictions, which can be seamlessly integrated into neural networks to ensure predictions align with certain symbolic constraints.
Overall, these works ensure that learned models follow specific symbolic rules, either through the design of objective functions or the modification of model architectures. However, while these approaches enforce rule compliance, the explicit meaning of individual model components often remains unclear, raising concerns about their transparency and interoperability.

\paragraph{Knowledge Compilation}
In the field of knowledge compilation~\citep{darwiche2002knowledge}, previous studies have made efforts to compile PGMs~\citep{knowledge_compile_pgm1, knowledge_compile_pgm2} or logical formulas~\citep{knowledge_compile_pgm2, knowledge_compile_formula1, knowledge_compile_formula2} into computationally efficient structures that support various inference tasks. In this paper, the proposed knowledge-injected parameter learning approach adopts a simple compilation strategy that compiles a set of weighted AND rules into a two-layer probabilistic circuit. Specifically, each product node corresponds to an individual AND rule, while the sum node encodes the weights associated with these rules.

\section{Conclusions} \label{sec:conclude}
In this paper, we propose Neural Probabilistic Circuits (\methodname s), a novel architecture that decomposes the decision-making process into attribute recognition and logical reasoning, enabling compositional and interpretable predictions. Experimental results across four image classification datasets demonstrate that \methodname~delivers competitive performance when compared with four baseline models.
In addition, we conduct a series of ablation studies and observe the following findings:
1) Compared with individual binary concepts, utilizing attributes preserves relational constraints, thus enhancing model performance.
2) Insufficient attributes compromise \methodname's performance on downstream tasks, and the impact of excluding specific attributes varies depending on their inherent properties.
3) For simple datasets, manually constructed circuits using the knowledge-injected approach are sufficient to result in competitive performance, whereas data-driven circuits are more effective on complex datasets.
4) Joint optimization leads to further improvements in the model performance.
Furthermore, we demonstrate how the provided most probable explanations and counterfactual explanations offer attribute-level insights into the model's decision-making process, thereby enhancing the human's understanding.
Finally, we discuss the limitations and potential future research directions for \methodname s from various perspectives, providing insights for future improvements.
Our work demonstrates the potential of \methodname s to enhance model interpretability and performance by integrating semantically meaningful attributes with probabilistic circuits, offering actionable insights for future advancements in interpretable machine learning through logical reasoning.

\paragraph{Acknowledgments}
We would like to extend our gratitude to Antonio Vergari for pointing out relevant literature on probabilistic circuits and knowledge compilation, and for the discussion on the relationship between \methodname s and semantic probabilistic layers.

\bibliographystyle{plainnat}
\bibliography{references}

\newpage
\begin{appendices}

\section{Joint Optimization} \label{sec:joint}
In this section, we provide additional algorithmic details for the proposed joint optimization approach. The approach aims to jointly optimize the attribute recognition model and the circuit-based task predictor in an \methodname~using the loss function defined in Equation (\ref{eq:joint}). Specifically, we employ the stochastic gradient descent algorithm to optimize the parameters of the attribute recognition model, \ie, $\theta$, and use the projected gradient descent algorithm to optimize the parameters of the circuit, \ie, $w$, to ensure their positivity.

\paragraph{Optimizing $\theta$}
Let $\eta_A$ denote the learning rate for the attribute recognition model. The parameters of the model at the $(t+1)$-th optimization step are as follows,

\resizebox{5.6in}{!}{%
\begin{minipage}{\textwidth}
\begin{align*}
\theta^{(t+1)} & = \theta^{(t)}+\eta_A \cdot\sum_{(x, y)\in D}\frac{\partial \log \operatorname{Pr}_{\theta^{(t)}, w^{(t)}} (Y=y \mid X=x)}{\partial \theta}.
\end{align*}
\end{minipage}
}
Here, $\operatorname{Pr}_{\theta^{(t)}, w^{(t)}}(Y=y \mid X=x)$ is computed using Equation (\ref{eq:inference}).

\paragraph{Optimizing $w$}
Let $\eta_C$ denote the learning rate for the circuit. The $d$-th weight at the $(t+1)$-th optimization step is as follows,

\resizebox{5.6in}{!}{%
\begin{minipage}{\textwidth}
\begin{align*}
w_d^{(t+1)} &= \mathcal{P}_{\mathbf{R}_{++}}\left\{w_d^{(t)}+\eta_C \cdot \sum_{(x, y)\in D}\frac{\partial \log \operatorname{Pr}_{\theta^{(t)}, w^{(t)}} (Y=y \mid X=x)}{\partial w_d}\right\}                                                               \\
                    &= \mathcal{P}_{\mathbf{R}_{++}}\left\{w_d^{(t)}+\eta_C \cdot \sum_{(x, y)\in D}\frac{1}{\operatorname{Pr}_{\theta^{(t)}, w^{(t)}} (Y=y \mid X=x)}\frac{\partial \operatorname{Pr}_{\theta^{(t)}, w^{(t)}} (Y=y \mid X=x)}{\partial w_d}\right\}.
\end{align*}
\end{minipage}
}
Here, $\Pr_{\theta^{(t)}, w^{(t)}} (Y=y \mid X=x)$ is computed using Equation (\ref{eq:inference}), and the gradient term is given by,

\resizebox{5.6in}{!}{%
\begin{minipage}{\textwidth}
\begin{align*}
&\frac{\partial \operatorname{Pr}_{\theta^{(t)}, w^{(t)}}\left(Y=y \mid X=x\right)}{\partial w_d}  = \sum_{a_{1:K}} \prod_{k=1}^K \operatorname{Pr}_{\theta^{(t)}}\left(A_k=a_k \mid X=x\right) \cdot \frac{\partial \operatorname{Pr}_{w^{(t)}}\left(Y=y \mid A_{1:K}=a_{1:K}\right)}{\partial w_d}       \\
         &= \sum_{a_{1:K}} \prod_{k=1}^K \operatorname{Pr}_{\theta^{(t)}}\left(A_k=a_k \mid X=x\right) \cdot \frac{f_S\left(y, a_{1:K}; w^{(t)}\right)}{f_S\left(\emptyset, a_{1:K}; w^{(t)}\right)} \cdot \left[
            \frac{\partial \log f_S\left(y, a_{1:K}; w^{(t)}\right)}{\partial w_d}-\frac{\partial \log f_S\left(\emptyset, a_{1:K}; w^{(t)}\right)}{\partial w_d} \right].
\end{align*}
\end{minipage}
}
In particular, $\prod_{k=1}^K \operatorname{Pr}_{\theta^{(t)}}\left(A_k=a_k \mid X=x\right)$ and $\frac{f_S\left(y, a_{1:K}; w^{(t)}\right)}{f_S\left(\emptyset, a_{1:K}; w^{(t)}\right)}$ are provided by the attribute recognition model and the circuit, respectively.
$\frac{\partial \log f_S\left(y, a_{1:K}; w^{(t)}\right)}{\partial w_d}$ and $\frac{\partial \log f_S\left(\emptyset, a_{1:K}; w^{(t)}\right)}{\partial w_d}$ can be computed recursively within a circuit \cite{spn}.

\section{Proof for Compositional Error} \label{sec:omit}
In this section, we present a detailed proof for Theorem \ref{thm:main-theorem}. Throughout the proof, expressions with parameters such as $\Pr_\theta$, $\Pr_w$, and $\Pr_{\theta, w}$ refer to probabilities learned by the models, while those without parameters represent ground-truth probabilities.

\resizebox{5.6in}{!}{%
\begin{minipage}{\textwidth}
\begin{align}
\epsilon_{\theta, w} &= \mathbb{E}_{X}\left[ \frac{1}{2}\sum_y \left|\Pr_{\theta, w}(Y=y\mid X) - \Pr(Y=y\mid X)\right| \right] \notag\\
&\leqslant \mathbb{E}_{X}[ \frac{1}{2}\sum_y \sum_{a_{1:K}} | \prod_k\Pr_\theta(A_k=a_k\mid X)\cdot\Pr_w(Y=y\mid A_{1:K}=a_{1:K}) - \prod_k\Pr(A_k=a_k\mid X)\cdot\Pr_w(Y=y\mid A_{1:K}=a_{1:K}) \notag\\
&\phantom{\leqslant\mathbb{E}_{X}[ \frac{1}{2}\sum_y \sum_{a_{1:K}}} + \prod_k\Pr(A_k=a_k\mid X)\cdot\Pr_w(Y=y\mid A_{1:K}=a_{1:K}) - \prod_k\Pr(A_k=a_k\mid X)\cdot\Pr(Y=y\mid A_{1:K}=a_{1:K}) | ] \notag\\
&\leqslant \mathbb{E}_X\left[ \frac{1}{2}\sum_y\sum_{a_{1:K}} \left| \prod_k\Pr_\theta(A_k=a_k\mid X)-\prod_k\Pr(A_k=a_k\mid X) \right|\cdot\Pr_w(Y=y\mid A_{1:K}=a_{1:K}) \right] \label{proof:first}\\
&\phantom{\quad} + \mathbb{E}_X\left[ \frac{1}{2}\sum_y\sum_{a_{1:K}} \prod_k\Pr(A_k=a_k\mid X)\cdot\left| \Pr_w(Y=y\mid A_{1:K}=a_{1:K})-\Pr(Y=y\mid A_{1:K}=a_{1:K}) \right|\right]. \label{proof:second}
\end{align}
\end{minipage}
}

Thus, the upper bound of $\epsilon_{\theta, w}$ is decomposed into two terms.

We derive the upper bound for the first term, \ie, Eq(\ref{proof:first}), as follows,

\resizebox{5.6in}{!}{%
\begin{minipage}{\textwidth}
\begin{align*}
\text{Eq}(\ref{proof:first}) &= \mathbb{E}_X\left[ \frac{1}{2}\sum_{a_{1:K}} \left| \prod_k\Pr_\theta(A_k=a_k\mid X)-\prod_k\Pr(A_k=a_k\mid X) \right| \right] = \mathbb{E}_X\left[ d_{\text{TV}}\left( \Pr_\theta(A_{1:K}\mid X), \Pr(A_{1:K}\mid X) \right) \right] = \epsilon_\theta\\
&\leqslant \mathbb{E}_X\left[ \frac{1}{2}\sum_k\sum_{a_k}\left| \Pr_\theta(A_k=a_k\mid X)-\Pr(A_k=a_k\mid X) \right| \right] = \sum_k\mathbb{E}_X\left[ d_{\text{TV}}\left( \Pr_\theta(A_k=a_k\mid X),\Pr(A_k=a_k\mid X) \right) \right] = \sum_k\epsilon_\theta^k.
\end{align*}
\end{minipage}
}

We derive the upper bound for the second term, \ie, Eq(\ref{proof:second}), as follows,

\resizebox{5.6in}{!}{%
\begin{minipage}{\textwidth}
\begin{align*}
\text{Eq}(\ref{proof:second}) &= \frac{1}{2}\sum_x\sum_y\sum_{a_{1:K}}\Pr(X=x, A_{1:K}=a_{1:K})\cdot\left| \Pr_w(Y=y\mid A_{1:K}=a_{1:K}) - \Pr(Y=y\mid A_{1:K}=a_{1:K}) \right| \\
&= \frac{1}{2}\sum_y\sum_{a_{1:K}}\Pr(A_{1:K}=a_{1:K})\cdot\left| \Pr_w(Y=y\mid A_{1:K}=a_{1:K}) - \Pr(Y=y\mid A_{1:K}=a_{1:K}) \right| \\
&= \frac{1}{2}\sum_y\sum_{a_{1:K}} | \Pr(A_{1:K}=a_{1:K})\cdot\Pr_w(Y=y\mid A_{1:K}=a_{1:K}) - \Pr_w(A_{1:K}=a_{1:K})\cdot\Pr_w(Y=y\mid A_{1:K}=a_{1:K}) \\
&\phantom{\frac{1}{2}\sum_y\sum_{a_{1:K}}\quad} + \Pr_w(A_{1:K}=a_{1:K})\cdot\Pr_w(Y=y\mid A_{1:K}=a_{1:K}) - \Pr(A_{1:K}=a_{1:K})\cdot\Pr(Y=y\mid A_{1:K}=a_{1:K})| \\
&\leqslant \frac{1}{2}\sum_y\sum_{a_{1:K}} \Pr_w(Y=y\mid A_{1:K}=a_{1:K})\cdot\left| \Pr(A_{1:K}=a_{1:K})-\Pr_w(A_{1:K}=a_{1:K}) \right| \\
&\phantom{\quad} +\frac{1}{2}\sum_y\sum_{a_{1:K}} \left| \Pr_w(Y=y, A_{1:K}=a_{1:K})-\Pr(Y=y, A_{1:K}=a_{1:K}) \right| \\
&= \frac{1}{2}\sum_{a_{1:K}} \left| \Pr(A_{1:K}=a_{1:K})-\Pr_w(A_{1:K}=a_{1:K}) \right| + d_{\text{TV}}\left( \Pr_w(Y, A_{1:K}), \Pr(Y, A_{1:K}) \right) \\ 
&= \frac{1}{2}\sum_{a_{1:K}} \left| \sum_y\Pr(Y=y, A_{1:K}=a_{1:K})-\sum_y\Pr_w(Y=y, A_{1:K}=a_{1:K}) \right| + d_{\text{TV}}\left( \Pr_w(Y, A_{1:K}), \Pr(Y, A_{1:K}) \right) \\
&\leqslant \frac{1}{2}\sum_{a_{1:K}}\sum_y \left| \Pr(Y=y, A_{1:K}=a_{1:K})-\Pr_w(Y=y, A_{1:K}=a_{1:K}) \right| + d_{\text{TV}}\left( \Pr_w(Y, A_{1:K}), \Pr(Y, A_{1:K}) \right) \\
&= 2d_{\text{TV}}\left( \Pr_w(Y, A_{1:K}), \Pr(Y, A_{1:K}) \right) = 2\epsilon_w.
\end{align*}
\end{minipage}
}

Combining results from Eq(\ref{proof:first}) and Eq(\ref{proof:second}), we have $\epsilon_{\theta, w} \leqslant \epsilon_\theta + 2\epsilon_w \leqslant \sum_k \epsilon_\theta^k + 2\epsilon_w $.

\section{Experimental Setup} \label{sec:more_setup}
In this section, we provide additional details regarding the experimental setup.

\paragraph{Model Architectures}
To ensure a fair comparison with the various baselines, we consistently adopt ResNet-34~\citep{resnet} as the backbone for all methods. Specifically, in CBM~\citep{cbm}, the concept recognition model is based on ResNet-34, where the final layer outputs concept probabilities. The task predictor is implemented as a linear layer that takes these concept probabilities as input and outputs predicted class scores.
For CEM~\citep{cem}, the first module is implemented using ResNet-34, with its final layer being the embedding layer defined in~\citep{cem}. This embedding layer produces both concept embeddings and concept probabilities. The subsequent task predictor uses only the concept embeddings as input to produce the predicted class scores.
In DCR~\citep{dcr}, the first module is identical to that of CEM, while the task predictor is the proposed deep concept reasoner (DCR). This reasoner takes concept probabilities as input and outputs predicted class scores, during which it leverages concept embeddings to formulate logical rules.
For \methodname, the attribute recognition model is implemented using a multi-task learning framework. Specifically, ResNet-34, without its original final layer, serves as the feature extractor to capture common features across multiple attributes. For each attribute, a series of two linear layers acts as a dedicated task head, outputting a probability vector corresponding to the attribute. The task predictor in \methodname~is implemented as a probabilistic circuit, with detailed construction provided in Section \ref{sec:relation_learning}.

\paragraph{Training Details}
For the baseline methods, training is performed with a batch size of 256 over 150 epochs. We use the SGD optimizer, configured with a learning rate of 0.01, a momentum of 0.9, and a weight decay of 0.00004. During training, we reduce the learning rate by a factor of 0.1 if the validation loss plateaus for 10 epochs. Additionally, the concept loss weight is set to 1 for CBM, CEM, and DCR, while the concept embedding size is configured to be 16 for both CEM and DCR~\citep{cem}.
For \methodname, the training process consists of three stages. In the first stage, the attribute recognition model is trained with a batch size of 256 over 150 epochs using the SGD optimizer. During training, we reduce the learning rate by a factor of 0.1 if the validation loss plateaus for 10 epochs. In the second stage, the circuit is learned and optimized following the procedures described by \citet{learnspn, cccp}, respectively. In the final stage, a learning rate of 0.01 is applied to the attribute recognition model while the circuit remains frozen, which is empirically shown to yield faster training and improved task performance.
We run \methodname s and all baseline models over five random seeds: 42, 52, 62, 72, and 82.

\paragraph{Counterfactual Explanations}
We adopt Algorithm \ref{alg:ce_optimization} to generate CEs over 100 iterations. The learning rate $\gamma$ is set to 0.005 for MNIST-Addition, 0.01 for GTSRB, 0.1 for CelebA, and 0.05 for AwA2.

\section{Attribute Configurations} \label{sec:configuration}
Table \ref{tab:attribute} provides an overview of the attributes used, along with their corresponding values, for the four benchmark datasets. In MNIST-Addition and GTSRB, attributes are \textit{single-valued}, meaning each sample is associated with only one value per attribute. In contrast, CelebA and AwA2 feature \textit{multi-valued} attributes, allowing a sample to be associated with multiple values for a single attribute. For example, an image from the ``zebra'' class has both ``black'' and ``white'' as color attributes. Furthermore, attribute annotations in GTSRB, CelebA, and AwA2 are \textit{class-wise}, where each class of samples all corresponds to the same specific set of attribute values. Conversely, in MNIST-Addition, samples from the same class may have different combinations of attribute values. For instance, a sample from the class ``8'' may correspond to attribute values such as ``0, 8'', ``1, 7'', or ``8, 0''. Table \ref{tab:samples} presents sample images from the four datasets.

\begin{table}[htbp]
  \centering
  \scriptsize
  \caption{Attributes and their corresponding values for the four benchmark datasets.}
  \label{tab:attribute}
  \resizebox{5.7in}{!}{%
  \begin{tabular}{llp{8cm}}
    \hline
    \textbf{Dataset} & \textbf{Attribute} & \textbf{Values} \\
    \hline
    \multirow{2}{*}{MNIST-Addition} 
    & Number-First & 0, 1, 2, 3, 4, 5, 6, 7, 8, 9 \\
    & Number-Second & 0, 1, 2, 3, 4, 5, 6, 7, 8, 9 \\
    \hline
    \multirow{4}{*}{GTSRB} 
    & Color & Blue, Red, White \\
    & Shape & Circle, Diamond, Octagon, Triangle \\
    & Symbol & Arrow-Consecutive-Turns, Arrow-Down-Left, Arrow-Down-Right, Arrow-Left, Arrow-Right, Arrow-Roundabout, Arrow-Up, Arrow-Up-and-Left, Arrow-Up-and-Right, Bicycle, Bump, Car, Car-Truck, Car-Two, Deer, Exclamation-Mark, Ice-or-Snow, Person, Person-Two, Road-Narrows, Roadworks, Slash, Text, Traffic-Signal, Truck, Undefined \\
    & Text & 100, 120, 20, 30, 50, 60, 70, 80, Stop, Undefined \\
    \hline
    \multirow{5}{*}{CelebA} 
    & Mouth & Mouth-Slightly-Open, None, Smiling \\
    & Face & High-Cheekbones, None, Oval-Face \\
    & Cosmetic & Heavy-Makeup, None, Wearing-Lipstick \\
    & Hair & None, Wavy-Hair \\
    & Appearance & Attractive, None \\
    \hline
    \multirow{4}{*}{AwA2} 
    & Color & Black, Blue, Brown, Gray, None, Orange, Red, White, Yellow \\
    & Surface & Furry, Hairless, Patches, Spots, Stripes, Toughskin \\
    & Body & Bipedal, Bulbous, Horns, Lean, Longleg, Longneck, Quadrupedal, Tail, Tusks \\
    & Limb & Claws, Flippers, Hands, Hooves, None, Pads, Paws \\
    \hline
  \end{tabular}}
\end{table}

\begin{table}[htbp]
  \centering
  \scriptsize
  \caption{Example images and their annotations from the four benchmark datasets. The CelebA image is redacted in compliance with its terms of use.}
  \label{tab:samples}
  \resizebox{5.7in}{!}{%
  \begin{tabular}{llp{7cm}}
    \hline
    \textbf{Image} & \textbf{Attribute} & \textbf{Values} \\
    \hline
    \multirow{2}{*}{\includegraphics[width=0.08\textwidth]{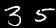}} 
    & Number-First & 3 \\
    & Number-Second & 5 \\
    \hline
    \multirow{4}{*}{\includegraphics[width=0.08\textwidth]{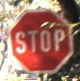}} 
    & Color & Red \\
    & Shape & Octagon \\
    & Symbol & Text \\
    & Text & Stop \\
    \hline
    \multirow{5}{*}{\includegraphics[width=0.08\textwidth]{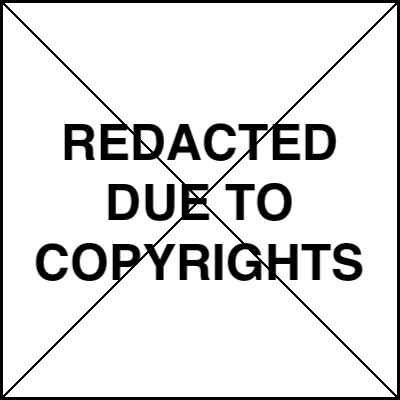}} 
    & Mouth & Mouth-Slightly-Open, Smiling \\
    & Face & High-Cheekbones \\
    & Cosmetic & Heavy-Makeup, Wearing-Lipstick \\
    & Hair & Wavy-Hair \\
    & Appearance & Attractive \\
    \hline
    \multirow{4}{*}{\includegraphics[width=0.08\textwidth]{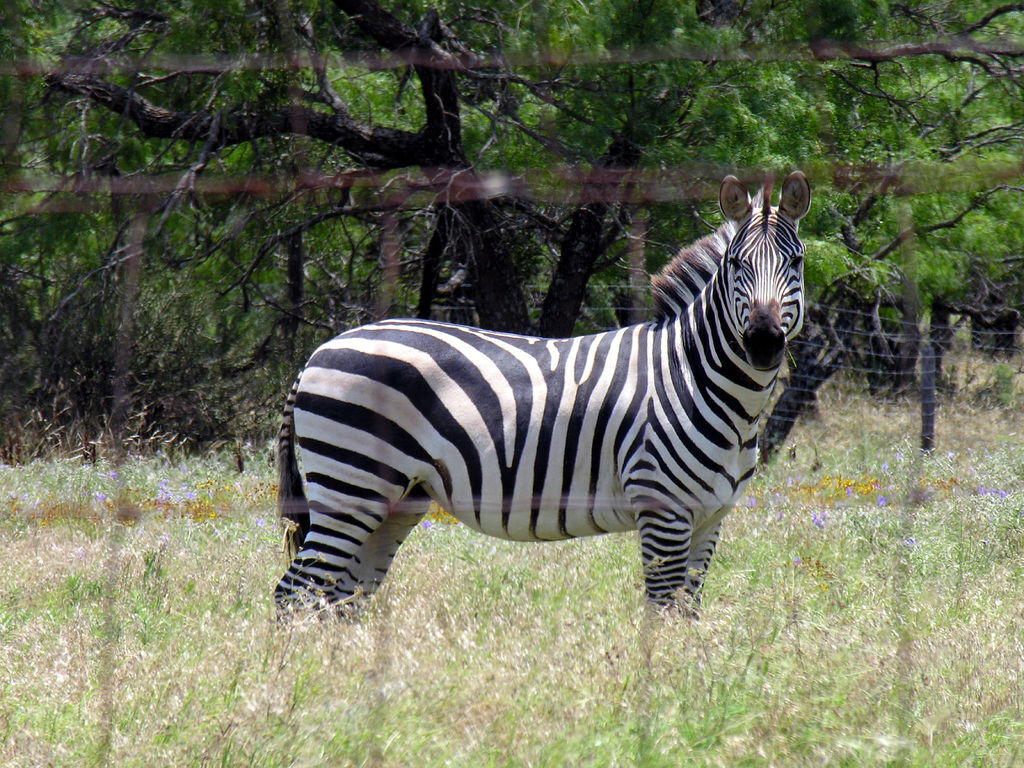}} 
    & Color & Black, White \\
    & Surface & Furry, Stripes, Toughskin \\
    & Body & Lean, Longleg, Longneck, Quadrupedal, Tail \\
    & Limb & Hooves \\
    \hline
  \end{tabular}}
\end{table}

\section{Human-Predefined Logical Rules} \label{sec:rule}
Human experience can be used to formulate specific logical rules for downstream tasks, representing valuable domain knowledge that can be integrated into models to enhance their reliability.
Here, we demonstrate a two-step procedure for constructing logical rules based on a given set of observed samples $\bar{D}=\{(x, y, a_{1:K})\}$.
1) For each observed sample $(x, y, a_{1:K})$, we establish a corresponding logical rule of the form $\mathbb{I}(A_1=a_1) \wedge \ldots \wedge \mathbb{I}(A_K=a_K) \wedge \mathbb{I}(Y=y)$.
2) The occurrences of the logical rules derived from all observed samples are then normalized to form a weight for each rule, ensuring that the rules reflect the distribution of the observed data. 
Table \ref{tab:rule} illustrates the rules constructed using training samples from the MNIST-Addition and GTSRB datasets.
In addition to the standardized procedure introduced above, humans can also leverage their expertise to incorporate more diverse and task-specific rules beyond this form.

\begin{table}[htbp]
  \centering
  \scriptsize
  \caption{Logical rules constructed using training samples from MNIST-Addition and GTSRB datasets (Part 1).}
  \label{tab:rule}%
  \resizebox{5.7in}{!}{%
    \begin{tabular}{lp{10cm}}
    \hline
    \textbf{Dataset} & \textbf{Logical Rules} \\
    \hline
    MNIST-Addition & $\mathbb{I}(\text{Number-First}=a_1) \wedge \mathbb{I}(\text{Number-Second}=a_2) \wedge \mathbb{I}(\text{Class}=a_1+a_2),~a_1,a_2\in[0,9]$ \\
    \hline
    \multirow{24}{*}{GTSRB} 
    & $\mathbb{I}(\text{Color}=\text{Red}) \wedge \mathbb{I}(\text{Shape}=\text{Circle}) \wedge \mathbb{I}(\text{Symbol}=\text{Text}) \wedge \mathbb{I}(\text{Text}=\text{20}) \wedge \mathbb{I}(\text{Class}=\text{regulatory--maximum-speed-limit-20})$ \\
    & $\mathbb{I}(\text{Color}=\text{Red}) \wedge \mathbb{I}(\text{Shape}=\text{Circle}) \wedge \mathbb{I}(\text{Symbol}=\text{Text}) \wedge \mathbb{I}(\text{Text}=\text{30}) \wedge \mathbb{I}(\text{Class}=\text{regulatory--maximum-speed-limit-30})$ \\
    & $\mathbb{I}(\text{Color}=\text{Red}) \wedge \mathbb{I}(\text{Shape}=\text{Circle}) \wedge \mathbb{I}(\text{Symbol}=\text{Text}) \wedge \mathbb{I}(\text{Text}=\text{50}) \wedge \mathbb{I}(\text{Class}=\text{regulatory--maximum-speed-limit-50})$ \\
    & $\mathbb{I}(\text{Color}=\text{Red}) \wedge \mathbb{I}(\text{Shape}=\text{Circle}) \wedge \mathbb{I}(\text{Symbol}=\text{Text}) \wedge \mathbb{I}(\text{Text}=\text{60}) \wedge \mathbb{I}(\text{Class}=\text{regulatory--maximum-speed-limit-60})$ \\
    & $\mathbb{I}(\text{Color}=\text{Red}) \wedge \mathbb{I}(\text{Shape}=\text{Circle}) \wedge \mathbb{I}(\text{Symbol}=\text{Text}) \wedge \mathbb{I}(\text{Text}=\text{70}) \wedge \mathbb{I}(\text{Class}=\text{regulatory--maximum-speed-limit-70})$ \\
    & $\mathbb{I}(\text{Color}=\text{Red}) \wedge \mathbb{I}(\text{Shape}=\text{Circle}) \wedge \mathbb{I}(\text{Symbol}=\text{Text}) \wedge \mathbb{I}(\text{Text}=\text{80}) \wedge \mathbb{I}(\text{Class}=\text{regulatory--maximum-speed-limit-80})$ \\
    & $\mathbb{I}(\text{Color}=\text{Red}) \wedge \mathbb{I}(\text{Shape}=\text{Circle}) \wedge \mathbb{I}(\text{Symbol}=\text{Text}) \wedge \mathbb{I}(\text{Text}=\text{100}) \wedge \mathbb{I}(\text{Class}=\text{regulatory--maximum-speed-limit-100})$ \\
    & $\mathbb{I}(\text{Color}=\text{Red}) \wedge \mathbb{I}(\text{Shape}=\text{Circle}) \wedge \mathbb{I}(\text{Symbol}=\text{Text}) \wedge \mathbb{I}(\text{Text}=\text{120}) \wedge \mathbb{I}(\text{Class}=\text{regulatory--maximum-speed-limit-120})$ \\
    & $\mathbb{I}(\text{Color}=\text{White}) \wedge \mathbb{I}(\text{Shape}=\text{Circle}) \wedge \mathbb{I}(\text{Symbol}=\text{Text}) \wedge \mathbb{I}(\text{Text}=\text{80}) \wedge \mathbb{I}(\text{Class}=\text{regulatory--end-of-maximum-speed-limit-80})$ \\
    & $\mathbb{I}(\text{Color}=\text{Red}) \wedge \mathbb{I}(\text{Shape}=\text{Circle}) \wedge \mathbb{I}(\text{Symbol}=\text{Cat-Two}) \wedge \mathbb{I}(\text{Text}=\text{Undefined}) \wedge \mathbb{I}(\text{Class}=\text{regulatory--no-overtaking})$ \\
    & $\mathbb{I}(\text{Color}=\text{Red}) \wedge \mathbb{I}(\text{Shape}=\text{Circle}) \wedge \mathbb{I}(\text{Symbol}=\text{Car-Truck}) \wedge \mathbb{I}(\text{Text}=\text{Undefined}) \wedge \mathbb{I}(\text{Class}=\text{regulatory--no-overtaking-by-heavy-goods-vehicles})$ \\
    & $\mathbb{I}(\text{Color}=\text{Red}) \wedge \mathbb{I}(\text{Shape}=\text{Triangle}) \wedge \mathbb{I}(\text{Symbol}=\text{Arrow-Up}) \wedge \mathbb{I}(\text{Text}=\text{Undefined}) \wedge \mathbb{I}(\text{Class}=\text{warning--crossroads})$ \\
    & $\mathbb{I}(\text{Color}=\text{White}) \wedge \mathbb{I}(\text{Shape}=\text{Diamond}) \wedge \mathbb{I}(\text{Symbol}=\text{Undefined}) \wedge \mathbb{I}(\text{Text}=\text{Undefined}) \wedge \mathbb{I}(\text{Class}=\text{regulatory--priority-road})$ \\
    & $\mathbb{I}(\text{Color}=\text{Red}) \wedge \mathbb{I}(\text{Shape}=\text{Triangle}) \wedge \mathbb{I}(\text{Symbol}=\text{Undefined}) \wedge \mathbb{I}(\text{Text}=\text{Undefined}) \wedge \mathbb{I}(\text{Class}=\text{regulatory--yield})$ \\
    & $\mathbb{I}(\text{Color}=\text{Red}) \wedge \mathbb{I}(\text{Shape}=\text{Octagon}) \wedge \mathbb{I}(\text{Symbol}=\text{Text}) \wedge \mathbb{I}(\text{Text}=\text{Stop}) \wedge \mathbb{I}(\text{Class}=\text{regulatory--stop})$ \\
    & $\mathbb{I}(\text{Color}=\text{Red}) \wedge \mathbb{I}(\text{Shape}=\text{Circle}) \wedge \mathbb{I}(\text{Symbol}=\text{Undefined}) \wedge \mathbb{I}(\text{Text}=\text{Undefined}) \wedge \mathbb{I}(\text{Class}=\text{regulatory--road-closed-to-vehicles})$ \\
    & $\mathbb{I}(\text{Color}=\text{Red}) \wedge \mathbb{I}(\text{Shape}=\text{Circle}) \wedge \mathbb{I}(\text{Symbol}=\text{Truck}) \wedge \mathbb{I}(\text{Text}=\text{Undefined}) \wedge \mathbb{I}(\text{Class}=\text{regulatory--no-heavy-goods-vehicles})$ \\
    & $\mathbb{I}(\text{Color}=\text{Red}) \wedge \mathbb{I}(\text{Shape}=\text{Circle}) \wedge \mathbb{I}(\text{Symbol}=\text{Slash}) \wedge \mathbb{I}(\text{Text}=\text{Undefined}) \wedge \mathbb{I}(\text{Class}=\text{regulatory--no-entry})$ \\
    & $\mathbb{I}(\text{Color}=\text{Red}) \wedge \mathbb{I}(\text{Shape}=\text{Triangle}) \wedge \mathbb{I}(\text{Symbol}=\text{Exclamation-Mark}) \wedge \mathbb{I}(\text{Text}=\text{Undefined}) \wedge \mathbb{I}(\text{Class}=\text{warning--other-danger})$ \\
    & $\mathbb{I}(\text{Color}=\text{Red}) \wedge \mathbb{I}(\text{Shape}=\text{Triangle}) \wedge \mathbb{I}(\text{Symbol}=\text{Arrow-Left}) \wedge \mathbb{I}(\text{Text}=\text{Undefined}) \wedge \mathbb{I}(\text{Class}=\text{warning--curve-left})$ \\
    & $\mathbb{I}(\text{Color}=\text{Red}) \wedge \mathbb{I}(\text{Shape}=\text{Triangle}) \wedge \mathbb{I}(\text{Symbol}=\text{Arrow-Right}) \wedge \mathbb{I}(\text{Text}=\text{Undefined}) \wedge \mathbb{I}(\text{Class}=\text{warning--curve-right})$ \\
    & $\mathbb{I}(\text{Color}=\text{Red}) \wedge \mathbb{I}(\text{Shape}=\text{Triangle}) \wedge \mathbb{I}(\text{Symbol}=\text{Arrow-Consecutive-Turns}) \wedge \mathbb{I}(\text{Text}=\text{Undefined}) \wedge \mathbb{I}(\text{Class}=\text{warning--double-curve-first-left})$ \\
    & $\mathbb{I}(\text{Color}=\text{Red}) \wedge \mathbb{I}(\text{Shape}=\text{Triangle}) \wedge \mathbb{I}(\text{Symbol}=\text{Bump}) \wedge \mathbb{I}(\text{Text}=\text{Undefined}) \wedge \mathbb{I}(\text{Class}=\text{warning--uneven-road})$ \\
    & $\mathbb{I}(\text{Color}=\text{Red}) \wedge \mathbb{I}(\text{Shape}=\text{Triangle}) \wedge \mathbb{I}(\text{Symbol}=\text{Car}) \wedge \mathbb{I}(\text{Text}=\text{Undefined}) \wedge \mathbb{I}(\text{Class}=\text{warning--slippery-road-surface})$ \\
    \hline
    \end{tabular}}
\end{table}%

\begin{table}[htbp]
  \centering
  \scriptsize
  \caption{Logical rules constructed using training samples from MNIST-Addition and GTSRB datasets (Part 2).}
  \label{tab:rule2}%
  \resizebox{5.7in}{!}{%
    \begin{tabular}{lp{10cm}}
    \hline
    \textbf{Dataset} & \textbf{Logical Rules} \\
    \hline
    \multirow{19}{*}{GTSRB} 
    & $\mathbb{I}(\text{Color}=\text{Red}) \wedge \mathbb{I}(\text{Shape}=\text{Triangle}) \wedge \mathbb{I}(\text{Symbol}=\text{Road-Narrows}) \wedge \mathbb{I}(\text{Text}=\text{Undefined}) \wedge \mathbb{I}(\text{Class}=\text{warning--road-narrows-right})$ \\
    & $\mathbb{I}(\text{Color}=\text{Red}) \wedge \mathbb{I}(\text{Shape}=\text{Triangle}) \wedge \mathbb{I}(\text{Symbol}=\text{Roadworks}) \wedge \mathbb{I}(\text{Text}=\text{Undefined}) \wedge \mathbb{I}(\text{Class}=\text{warning--roadworks})$ \\
    & $\mathbb{I}(\text{Color}=\text{Red}) \wedge \mathbb{I}(\text{Shape}=\text{Triangle}) \wedge \mathbb{I}(\text{Symbol}=\text{Traffic-Signal}) \wedge \mathbb{I}(\text{Text}=\text{Undefined}) \wedge \mathbb{I}(\text{Class}=\text{warning--traffic-signals})$ \\
    & $\mathbb{I}(\text{Color}=\text{Red}) \wedge \mathbb{I}(\text{Shape}=\text{Triangle}) \wedge \mathbb{I}(\text{Symbol}=\text{Person}) \wedge \mathbb{I}(\text{Text}=\text{Undefined}) \wedge \mathbb{I}(\text{Class}=\text{warning--pedestrians-crossing})$ \\
    & $\mathbb{I}(\text{Color}=\text{Red}) \wedge \mathbb{I}(\text{Shape}=\text{Triangle}) \wedge \mathbb{I}(\text{Symbol}=\text{Person-Two}) \wedge \mathbb{I}(\text{Text}=\text{Undefined}) \wedge \mathbb{I}(\text{Class}=\text{warning--children})$ \\
    & $\mathbb{I}(\text{Color}=\text{Red}) \wedge \mathbb{I}(\text{Shape}=\text{Triangle}) \wedge \mathbb{I}(\text{Symbol}=\text{Bicycle}) \wedge \mathbb{I}(\text{Text}=\text{Undefined}) \wedge \mathbb{I}(\text{Class}=\text{warning--bicycles-crossing})$ \\
    & $\mathbb{I}(\text{Color}=\text{Red}) \wedge \mathbb{I}(\text{Shape}=\text{Triangle}) \wedge \mathbb{I}(\text{Symbol}=\text{Ice-or-Snow}) \wedge \mathbb{I}(\text{Text}=\text{Undefined}) \wedge \mathbb{I}(\text{Class}=\text{warning--ice-or-snow})$ \\
    & $\mathbb{I}(\text{Color}=\text{Red}) \wedge \mathbb{I}(\text{Shape}=\text{Triangle}) \wedge \mathbb{I}(\text{Symbol}=\text{Deer}) \wedge \mathbb{I}(\text{Text}=\text{Undefined}) \wedge \mathbb{I}(\text{Class}=\text{warning--wild-animals})$ \\
    & $\mathbb{I}(\text{Color}=\text{White}) \wedge \mathbb{I}(\text{Shape}=\text{Circle}) \wedge \mathbb{I}(\text{Symbol}=\text{Undefined}) \wedge \mathbb{I}(\text{Text}=\text{Undefined}) \wedge \mathbb{I}(\text{Class}=\text{regulatory--end-of-prohibition})$ \\
    & $\mathbb{I}(\text{Color}=\text{Blue}) \wedge \mathbb{I}(\text{Shape}=\text{Circle}) \wedge \mathbb{I}(\text{Symbol}=\text{Arrow-Right}) \wedge \mathbb{I}(\text{Text}=\text{Undefined}) \wedge \mathbb{I}(\text{Class}=\text{regulatory--turn-right-ahead})$ \\
    & $\mathbb{I}(\text{Color}=\text{Blue}) \wedge \mathbb{I}(\text{Shape}=\text{Circle}) \wedge \mathbb{I}(\text{Symbol}=\text{Arrow-Left}) \wedge \mathbb{I}(\text{Text}=\text{Undefined}) \wedge \mathbb{I}(\text{Class}=\text{regulatory--turn-left-ahead})$ \\
    & $\mathbb{I}(\text{Color}=\text{Blue}) \wedge \mathbb{I}(\text{Shape}=\text{Circle}) \wedge \mathbb{I}(\text{Symbol}=\text{Arrow-Up}) \wedge \mathbb{I}(\text{Text}=\text{Undefined}) \wedge \mathbb{I}(\text{Class}=\text{regulatory--go-straight})$ \\
    & $\mathbb{I}(\text{Color}=\text{Blue}) \wedge \mathbb{I}(\text{Shape}=\text{Circle}) \wedge \mathbb{I}(\text{Symbol}=\text{Arrow-Up-and-Right}) \wedge \mathbb{I}(\text{Text}=\text{Undefined}) \wedge \mathbb{I}(\text{Class}=\text{regulatory--go-straight-or-turn-right})$ \\
    & $\mathbb{I}(\text{Color}=\text{Blue}) \wedge \mathbb{I}(\text{Shape}=\text{Circle}) \wedge \mathbb{I}(\text{Symbol}=\text{Arrow-Up-and-Left}) \wedge \mathbb{I}(\text{Text}=\text{Undefined}) \wedge \mathbb{I}(\text{Class}=\text{regulatory--go-straight-or-turn-left})$ \\
    & $\mathbb{I}(\text{Color}=\text{Blue}) \wedge \mathbb{I}(\text{Shape}=\text{Circle}) \wedge \mathbb{I}(\text{Symbol}=\text{Arrow-Down-Right}) \wedge \mathbb{I}(\text{Text}=\text{Undefined}) \wedge \mathbb{I}(\text{Class}=\text{regulatory--keep-right})$ \\
    & $\mathbb{I}(\text{Color}=\text{Blue}) \wedge \mathbb{I}(\text{Shape}=\text{Circle}) \wedge \mathbb{I}(\text{Symbol}=\text{Arrow-Down-Left}) \wedge \mathbb{I}(\text{Text}=\text{Undefined}) \wedge \mathbb{I}(\text{Class}=\text{regulatory--keep-left})$ \\
    & $\mathbb{I}(\text{Color}=\text{Blue}) \wedge \mathbb{I}(\text{Shape}=\text{Circle}) \wedge \mathbb{I}(\text{Symbol}=\text{Arrow-Roundabout}) \wedge \mathbb{I}(\text{Text}=\text{Undefined}) \wedge \mathbb{I}(\text{Class}=\text{regulatory--roundabout})$ \\
    & $\mathbb{I}(\text{Color}=\text{White}) \wedge \mathbb{I}(\text{Shape}=\text{Circle}) \wedge \mathbb{I}(\text{Symbol}=\text{Car-Two}) \wedge \mathbb{I}(\text{Text}=\text{Undefined}) \wedge \mathbb{I}(\text{Class}=\text{regulatory--end-of-no-overtaking})$ \\
    & $\mathbb{I}(\text{Color}=\text{White}) \wedge \mathbb{I}(\text{Shape}=\text{Circle}) \wedge \mathbb{I}(\text{Symbol}=\text{Car-Truck}) \wedge \mathbb{I}(\text{Text}=\text{Undefined}) \wedge \mathbb{I}(\text{Class}=\text{regulatory--end-of-no-overtaking-by-heavy-goods-vehicles})$ \\
    \hline
    \end{tabular}}
\end{table}%

\clearpage
\section{Model Explanations} \label{sec:more_explanation}
Figure \ref{fig:mpe_manual} illustrates examples for \methodname (Knowledge) from the four benchmark datasets. Each example comprises an image, the ground-truth labels for the class and attributes, the class predicted by \methodname (Knowledge), and, lastly, the corresponding MPE. In these instances, \methodname (Knowledge) provides correct class predictions, and the MPEs are aligned with the ground-truth attribute labels.

\begin{figure}[htbp]
    \centering
    \includegraphics[width=5.7in]{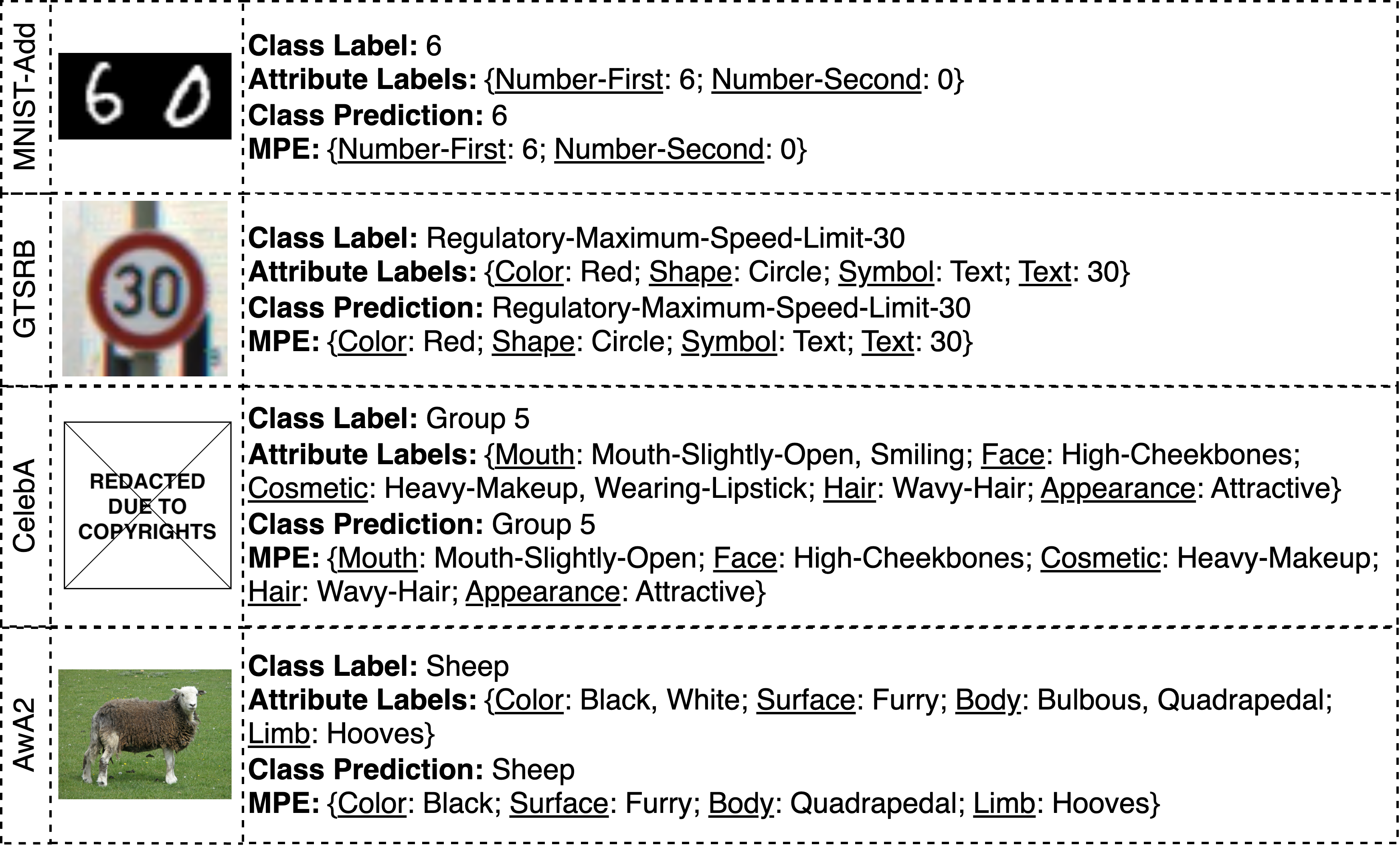}
    \caption{Examples from the four benchmark datasets. Each example includes an image, the ground-truth class label, the ground-truth attribute labels, the \methodname (Knowledge) class prediction, and, lastly, the corresponding MPE. These examples illustrate MPEs that align with the ground-truth attribute labels. The CelebA image is redacted in compliance with its terms of use.}
    \label{fig:mpe_manual}
\end{figure}

Figure \ref{fig:ce_manual} illustrates some examples for \methodname (Knowledge) from the four benchmark datasets. Each sample consists of an image, the attribute and class predictions, the generated CE, and, lastly, the class prediction corrected by the CE.
In these examples, \methodname (Knowledge) incorrectly predicts the class, while CEs effectively correct them by making minimal adjustments to the attribute predictions.

\begin{figure}[htbp]
    \centering
    \includegraphics[width=5.7in]{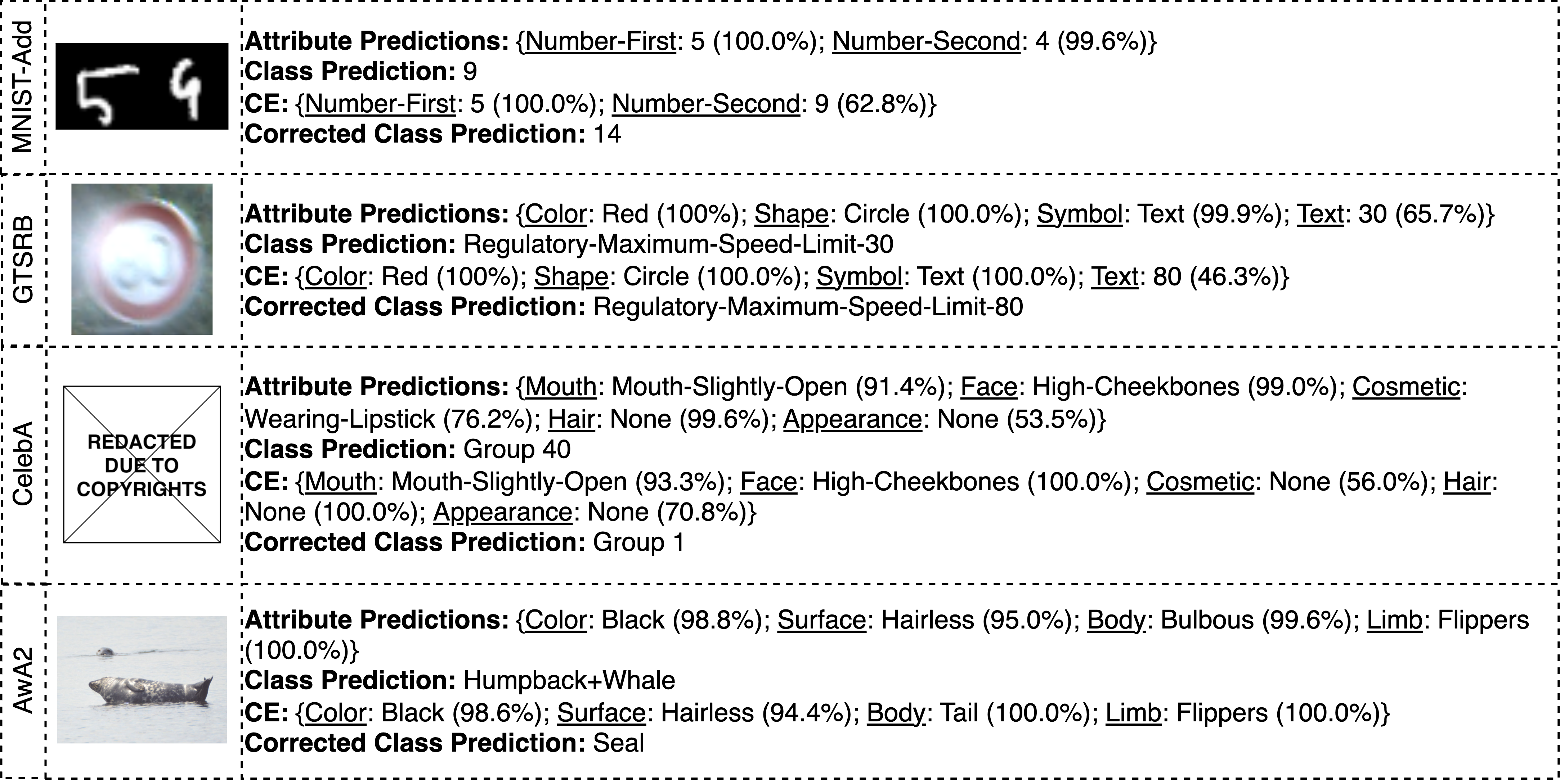}
    \caption{Examples from the four benchmark datasets. Each example includes an image, the \methodname (Knowledge) attribute and class predictions, the corresponding CE, and the class prediction corrected by CE. These examples focus on CEs that successfully correct the class predictions. For simplicity, both attribute predictions and CEs display only the most confident value for each attribute. The CelebA image is redacted in compliance with its terms of use.}
    \label{fig:ce_manual}
\end{figure}




\end{appendices}


\end{document}